\documentclass{article}
\usepackage{log_2024}						

\usepackage{amsmath,amsfonts,bm,dsfont,diagbox,amsthm}
\usepackage{tikz}
\usetikzlibrary{arrows.meta, positioning, shapes.geometric, fit}
\usepackage[inline]{enumitem}
\usepackage{thmtools} 
\usepackage{thm-restate}

\makeatletter
\let\save@mathaccent\mathaccent
\newcommand*\if@single[3]{%
    \setbox0\hbox{${\mathaccent"0362{#1}}^H$}%
    \setbox2\hbox{${\mathaccent"0362{\kern0pt#1}}^H$}%
    \ifdim\ht0=\ht2 #3\else #2\fi
    }
\newcommand*\rel@kern[1]{\kern#1\dimexpr\macc@kerna}
\newcommand*\widebar[1]{{\@ifnextchar^{{\wide@bar{#1}{0}}}{\wide@bar{#1}{1}}}}
\newcommand*\wide@bar[2]{\if@single{#1}{\wide@bar@{#1}{#2}{1}}{\wide@bar@{#1}{#2}{2}}}
\newcommand*\wide@bar@[3]{%
\begingroup
\def\mathaccent##1##2{%
    \let\mathaccent\save@mathaccent
    \if#32 \let\macc@nucleus\first@char \fi
    \setbox\z@\hbox{$\macc@style{\macc@nucleus}_{}$}%
    \setbox\tw@\hbox{$\macc@style{\macc@nucleus}{}_{}$}%
    \dimen@\wd\tw@
    \advance\dimen@-\wd\z@
    \divide\dimen@ 3
    \@tempdima\wd\tw@
    \advance\@tempdima-\scriptspace
    \divide\@tempdima 10
    \advance\dimen@-\@tempdima
    \ifdim\dimen@>\z@ \dimen@0pt\fi
    \rel@kern{0.6}\kern-\dimen@
    \if#31
        \overline{\rel@kern{-0.6}\kern\dimen@\macc@nucleus\rel@kern{0.4}\kern\dimen@}%
        \advance\dimen@0.4\dimexpr\macc@kerna
        \let\final@kern#2%
        \ifdim\dimen@<\z@ \let\final@kern1\fi
        \if\final@kern1 \kern-\dimen@\fi
    \else
        \overline{\rel@kern{-0.6}\kern\dimen@#1}%
    \fi
}%
\macc@depth\@ne
\let\math@bgroup\@empty \let\math@egroup\macc@set@skewchar
\mathsurround\z@ \frozen@everymath{\mathgroup\macc@group\relax}%
\macc@set@skewchar\relax
\let\mathaccentV\macc@nested@a
\if#31
    \macc@nested@a\relax111{#1}%
\else
    \def\gobble@till@marker##1\endmarker{}%
    \futurelet\first@char\gobble@till@marker#1\endmarker
    \ifcat\noexpand\first@char A\else
        \def\first@char{}%
    \fi
    \macc@nested@a\relax111{\first@char}%
\fi
\endgroup
}
\makeatother

\newtheorem{proposition}{Proposition}

\def\eqref#1{equation~\ref{#1}}

\def\1{\bm{1}}

\def\mA{{\bm{A}}}

\def\mE{{\bm{E}}}

\def\mX{{\bm{X}}}

\def\mZ{{\bm{Z}}}

\DeclareMathAlphabet{\mathsfit}{\encodingdefault}{\sfdefault}{m}{sl}
\SetMathAlphabet{\mathsfit}{bold}{\encodingdefault}{\sfdefault}{bx}{n}

\newcommand{\revision}[1]{{\color{black}#1}}
\newcommand{\revisionlog}[1]{{\color{black}#1}}

\usepackage{booktabs}						
\usepackage{multirow}						
\usepackage{amsfonts}						
\usepackage{graphicx}						
\usepackage{duckuments}						

\usepackage[numbers,compress,sort]{natbib}	

\usepackage[utf8]{inputenc} 
\usepackage[T1]{fontenc}    
\usepackage[backref=page]{hyperref}       
\usepackage{url}            
\usepackage{multirow}
\usepackage{booktabs}       
\usepackage{amsfonts}       
\usepackage{nicefrac}       
\usepackage{microtype}      
\usepackage{xcolor}         
\usepackage{xspace}
\usepackage{adjustbox}
\usepackage{xcolor}
\usepackage{wrapfig}

\usepackage{bbm}
\usepackage{CJKutf8}

\usepackage{cleveref}
\usepackage{subcaption}

\usepackage{algpseudocode}
\usepackage{algorithm}

\newcommand{\ourmethod}{TRIX\xspace}

\title{TRIX: A More Expressive \revision{Model for Zero-shot\\ 
Domain Transfer in Knowledge Graphs}}

\author[Y. Zhang et al.]{%
Yucheng Zhang\\
Purdue University\\
\email{zhan4332@purdue.edu}\And
Beatrice Bevilacqua\\
Purdue University\\
\email{bbevilac@purdue.edu}\AND
Mikhail Galkin\\
Intel AI Lab\\
\email{mikhail.galkin@intel.com}\And
Bruno Ribeiro\\
Purdue University\\
\email{ribeirob@purdue.edu}
}

\begin{document}

\maketitle

\begin{abstract}
\revision{Fully inductive knowledge graph models can be trained on multiple domains and subsequently perform zero-shot knowledge graph completion (KGC) in new unseen domains.
This is an important capability towards the goal of having foundation models for knowledge graphs.}
  In this work, we introduce a more expressive and capable \revision{fully inductive model}, dubbed \ourmethod, which not only yields strictly more expressive triplet embeddings (head entity, relation, tail entity)  compared to state-of-the-art methods, but also introduces a new capability: directly handling both entity and relation prediction tasks in inductive settings. 
  Empirically, we show that \ourmethod outperforms the state-of-the-art fully inductive models in zero-shot entity and relation predictions in new domains, and outperforms large-context LLMs in out-of-domain predictions.
  The source code is available at \revision{\url{https://github.com/yuchengz99/TRIX}}.
\end{abstract}

\section{Introduction}
\revision{
Fully inductive knowledge graph models can perform zero-shot Knowledge Graph Completion (KGC), which predicts missing facts in entirely new domains that were not part of the training data. This is particularly challenging because these new domains may contain just unseen relation types and new entities~\citep{galkin2023towards,gao2023double,lee2023ingram}.
} %
Fully inductive models are trained on one or multiple Knowledge Graphs (\revision{KG}) with a specific set of relations. Previous work has emphasized the importance of double-equivariance~\cite{gao2023double}, that is, equivariance to permutations of both entity ids and relation ids, as a fundamental property that \revision{fully inductive models} must have to transfer across KG domains. Intuitively, this equivariance allows \revision{models} to focus on the underlying structural invariances, despite semantic differences and variations in identifiers across different KGs.

However, despite the notable achievements of current \revision{fully inductive models}, several open challenges remain, including:
\begin{enumerate*}[label=(\arabic*)]
    \item Limited expressivity of existing methods;
    \item Insufficient support for relation prediction tasks; and
    \item Underexploration of the abilities of Large Language Models (LLMs) to perform the same tasks.
\end{enumerate*}
More precisely, existing state-of-the-art \revision{fully inductive models}, such as ULTRA~\cite{galkin2023towards}, have expressivity limitations, as we show in \Cref{sec:expressivity}, which implies that certain non-isomorphic triplets inevitably get the same representations, and therefore necessarily the same predictions despite their differences. Since increased expressive power tend to translate into better downstream performances in traditional graph tasks~\citep{bouritsas2020improving,bevilacqua2022equivariant,zhang2023rethinking,puny2023equivariant}, an open question is whether improving the expressivity of \revision{fully inductive models} would also improve their empirical performance. Additionally, existing \revision{fully inductive models} are primarily designed for entity prediction tasks, answering queries such as (?, relation, tail entity) or (head entity, relation, ?), where the goal is to predict the head or tail entity of a given triplet. Consequently, they miss the equally important relation prediction tasks, namely queries such as (head entity, ?, tail entity), where the goal is to predict the missing relation between two entities~\citep{teru2020inductive,cui2021type,liang2024mines,su2024anchoring}. 
Finally, while large-context LLMs have demonstrated notable performance in KGC tasks~\citep{zhang2023making,chen2023dipping,shu2024knowledge,xu2024multi,wei2024kicgpt}, their effectiveness in the inductive setting of our interest, where test KGs come from new domains, remains largely underexplored. Therefore, an evaluation and comparison with \revision{fully inductive graph models} is needed to understand whether \revision{fully inductive graph models} are necessary or if repurposing LLMs would suffice.

\begin{figure}[t]
\centering
\includegraphics[width=1.0\textwidth]{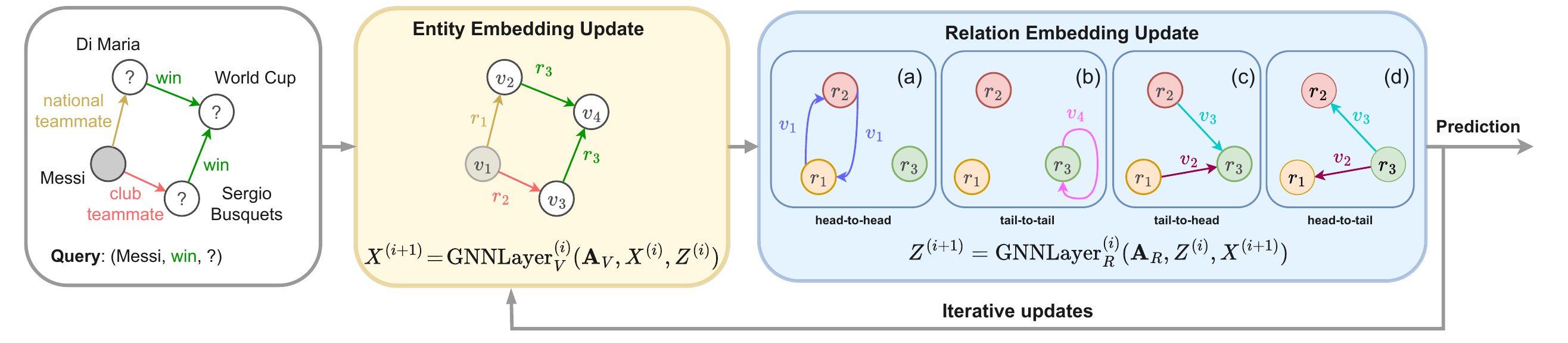}
\caption{Overview of \ourmethod showcasing how a KG in a given domain is represented in \ourmethod's double-equivariant architecture through entity embedding and relation embedding iterative updates.}
\vspace{-10pt}
\label{fig:ourmethod}
\end{figure}

\textbf{Our approach.} In this paper, we aim to address the open challenges in \revision{fully inductive models}. We first show that the limited expressive power of the state-of-the-art \revision{fully inductive models} ULTRA~\citep{galkin2023towards} arises from its approach in capturing relation interactions, obtained by counting the number of entities sharing a pair of relations, rather than \emph{which} entities share those relations. Then, we propose \ourmethod (\underline{T}ransferable \underline{R}elation-Entity \underline{I}nteractions in crossing patterns (\underline{X}-patterns)), which we show to return \emph{strictly more expressive} triplet representations than existing methods.  As illustrated in \Cref{fig:ourmethod}, given any input KG, \ourmethod first constructs a graph of relations, where each node is a relation from the original graph and edges denote shared entities among those relations. Then, it refines relation and entity representations by iteratively applying graph neural network layers over the graph of relations and the original graph. In this way, \ourmethod obtains representations of relation and entities directly applicable to zero-shot tasks. 

We demonstrate that, by design, \ourmethod efficiently handles relation prediction tasks, a capability lacking in existing state-of-the-art \revision{fully inductive models}. In particular, \ourmethod can answer relation prediction queries in a single forward pass, while existing \revision{fully inductive models} require performing a number of forward passes equal to the number of relations, for a single relation prediction query. 

Finally, we explore the capabilities of LLMs for KGC by designing a comprehensive set of experiments. We demonstrate that, while LLMs can do KGC accurately given enough context about the background knowledge, they rely on the textual information (and its semantics), and therefore fail to utilize the actual graph information, given in the context. We show that this result has several implications, including failure cases when the relation names are not given due to privacy concerns, or simply when they are not known by the LLMs.

\textbf{Contributions.} Our key contributions are as follows:
\begin{enumerate*}[label=(\arabic*), leftmargin=*]
\item We propose a novel \revision{fully inductive model on KGs}, \ourmethod, which exhibits greater expressive power compared to prior methods and can handle relation predictions efficiently; 
\item We show that the increased expressiveness of \ourmethod allows it to surpass state-of-the-art methods in 57 KG datasets in both entity and relation predictions;
\item We present an experimental study of LLMs on the same tasks and show that existing LLMs have limited capabilities in exploiting graph information, which are needed to perform tasks on new domains.
\end{enumerate*}

\section{Related Work}
\textbf{Fully Inductive Models over KGs.}
Original efforts in \revision{inductive} Knowledge Graph Completion primarily focused on handling new entities at test time, \revisionlog{but not new relations}~\cite{schlichtkrull2018modeling,teru2020inductive,zhu2021neural,yang2017differentiable, sadeghian2019drum,qu2020rnnlogic, cheng2022rlogic, shengyuan2024differentiable, liu2021indigo, tena2022explainable}. However, emerging methodologies tackle inductive learning scenarios with both new entities and new relations in test~\cite{gao2023double,lee2023ingram,galkin2023towards}.
\citet{gao2023double} approaches the problem by treating relations as set elements and employing DSS layers~\citep{maron2020learning} to learn representations equivariant to permutations of both node and relation ids. 
On the other hand, \citet{lee2023ingram} and \citet{galkin2023towards} introduce relation graphs to capture relation representations based on their interactions. In this work, we extend the latter approach by proposing a novel design for relation graphs, which allows to capture more expressive structural patterns in KGs.

\textbf{Knowledge Graph Completion with Large Language Models.} Despite the impressive capabilities of pretrained LLMs, their application to KGC has primarily focused on leveraging textual information rather than exploiting the underlying graph structure~\cite{chen2023dipping}. Recent efforts aims to enhance LLM performance in KGC by incorporating neighborhood or path information between entities as part of prompts in an in-context learning approach~\cite{wei2024kicgpt, xu2024multi, shu2024knowledge}. However, these approaches have mainly been applied in transductive settings.
\citet{li2024condensed} investigated zero-shot link prediction tasks with LLMs but only considered scenarios involving new relations, not new entities. The effectiveness of LLMs in KGC tasks with both new entities and new relations remains unexplored.
Furthermore, the extent to which LLMs exhibit double-equivariance, a crucial property for inductive reasoning on KGs \cite{gao2023double}, has yet to be thoroughly investigated. In this work, we address these gaps by conducting comprehensive experiments to evaluate the double-equivariance property of LLMs in KGC tasks.


\section{Preliminaries}\label{sec:preliminary}
A Knowledge Graph $G$ is a tuple $(V, E, R)$ where $V$ is a finite set of entities, $R$ is a finite set of relations and $E \in (V \times R \times V)$ is a finite set of edges representing relations between entities. We denote by $G_{\text{train}}=(V_{\text{train}}, E_{\text{train}}, R_{\text{train}})$ the training graph, and by $G_{\text{inf}}=(V_{\text{inf}}, E_{\text{inf}}, R_{\text{inf}})$ the test (or inference) graph. \revision{Since we focus on inductive settings, where the test graph comes from a different domain, the training and the test graphs have disjoint entity and relation sets, that is, $V_{\text{inf}} \not\subset V_{\text{train}}$ and $R_{\text{inf}} \not\subset R_{\text{train}}$.} Despite being unseen, it is however assumed that the test graph shares certain structural patterns with the training graph. These structural similarities imply the existence of invariances sufficient for accurate predictions. Consequently, models trained on the training graph can transfer the learned knowledge to the test graph, leveraging the shared structural patterns for effective predictions.

Since KGs are often incomplete~\citep{rossi2021knowledge}, KGC has been widely utilized to infer missing information by predicting missing triplets. KGC comprises two key tasks: entity prediction\cite{bordes2013translating} addresses queries $(h, r, ?)$ which predicts the tail entity given a head entity $h$ and a relation $r$ (or, equivalently, predict the head); relation prediction\cite{teru2020inductive}, on the other hand, focuses on queries $(h, ?, t)$, aiming to predict the existence of a link between a head entity $h$ and a tail entity $t$, and determining the relation type.

To perform entity and relation predictions in inductive settings, \citet{gao2023double} recently identified the concept of double-equivariance, i.e., equivariance to both entity and relation ids permutations, as a necessary property that \revision{fully inductive models} must possess. \revision{This property ensures that the architecture does not use relation and entity IDs, but, instead, captures the interactions among them. Indeed, even if relations and entities vary across datasets, the interactions between them may be similar and transferable.} Existing \revision{fully inductive models} achieve double equivariance either by designing architectures that are intrinsically equivariant to both permutation groups~\citep{gao2023double}, or by constructing a relation graph that captures relation interactions regardless of their ids or semantics~\citep{lee2023ingram,galkin2023towards, gao2023double}. Since the latter tends to be a lighter approach, which also yields better results in practice, we will in the following focus on that. In particular, these methods first construct a relation graph where entities are relations and edges represent the number of times two relations share an entity. This relation graph is then passed to a standard graph neural network with labeling trick~\citep{zhang2021labeling} to obtain relation representations conditioned on the query of interest, which can directly be used in inductive settings. We show how \ourmethod extends this approach and constructs a relation graph that does not merely count the number of entities shared among two relations, but identify which entities are shared, an approach that we show to return strictly more expressive representations in \Cref{sec:expressivity}.

Finally, we remark that existing \revision{fully inductive models} typically begin by applying the labeling trick~\citep{zhang2021labeling} to obtain initial relative relation representations, conditioned on the query relation~\citep{galkin2023towards}. However, in the relation prediction ($h$, ?, $t$), there is no specific query relation available for conditioning, as this is unknown and constitutes the target of our query. As a result, existing methods must convert relation prediction into a triplet ranking problem, evaluating the possibility of ($h$, $r$, $t$) for all relations $r$ in the relation set. This implies that the model needs to perform one forward pass for each relation. 
In contrast,
we show in the next section that \ourmethod can perform the task in one single forward pass. \label{preliminary}

\section{\ourmethod Framework}
\revision{As discussed in previous sections, existing fully inductive models suffer from limited expressivity due to the way in which they construct the relation graph and inefficiency in relation prediction tasks due to the message passing scheme. We fill in these gaps by proposing the framework of \ourmethod, in which the design of our new relation graph records the entity property to enhance expressive power and to the design of iterative message passing mechanism makes relation prediction in one forward pass possible.}
Then we explore \ourmethod's theoretical properties, such as expressiveness and time complexity. An overview of \ourmethod is illustrated in \Cref{fig:ourmethod}.

\subsection{Relation Adjacency Matrix}

\revision{Existing methods have limited expressive power. For instance, the relation graphs of both ULTRA~\citep{galkin2023towards} and InGram~\citep{lee2023ingram} are too invariant: In \Cref{fig:example} we show how the edges of the relation graphs of these existing methods (how many common entities two relations have) are not expressive enough for the entity prediction task. We increase expressiveness by including entity information, that is, edges between two relations contain information of which entities share these relations.
}
Given the graph $G$, with adjacency matrix $\mA_V \in \mathbb{R}^{|V| \times |V| \times |R|}$ (entity adjacency matrix), \ourmethod first constructs a relation graph $G_R$ with adjacency matrix $\mA_R$ (relation adjacency matrix). Entities in $\mA_R$ represent the relations in $\mA_V$, and edges in $\mA_R$ denote entities (in $\mA_V$) that share two relations (in $\mA_V$).
Specifically, for any pair of relations $r_i, r_j \in R$ in the original graph $G$, we count how many times each entity $v_k \in V$ is part of triplets involving these two relations as:
\begin{enumerate*}[label=(\arabic*)]
    \item the head entity in both (that is, how many triplets like ($v_k$, $r_i$, $\star$) and ($v_k$, $r_j$, $\star$) exist, where $\star$ is a placeholder for entities),
    \item the tail entity in both (that is ($\star$, $r_i$, $v_k$) and ($\star$, $r_j$, $v_k$)),
    \item the head in the first and the tail in the other
    (that is ($v_k$, $r_i$, $\star$) and ($\star$, $r_j$, $v_k$)), or
    \item vice-versa
    (that is ($\star$, $r_i$, $v_k$) and ($v_k$, $r_j$, $\star$)).
\end{enumerate*}

Since we keep these four roles (head-head, tail-tail, head-tail, tail-head) separate, and we count the above for each entity $v \in V$, the relation adjacency matrix $\mA_R$ \revision{is a tensor of shape} $|R| \times |R| \times |V| \times 4$. Note that this is in contrast with previous methods~\citep{galkin2023towards,lee2023ingram}, that, despite also maintaining the distinct roles, do not differentiate which entities participate in the role but only how many, a choice that impacts the expressive power as we shall see next.

Mathematically, we first construct two matrices $\mE_h \in \mathbb{R}^{|V| \times |R|}$ and $\mE_t \in \mathbb{R}^{|V| \times |R|}$ capturing how many times each entity 
$v \in V$ is the head of triplets involving relation $r \in R$, or the tail, respectively.
Then, we construct four different intermediate relation adjacency matrices capturing the four roles $\mA_R^{hh}$, $\mA_R^{tt}$, $\mA_R^{ht}$, and $\mA_R^{th} \in \mathbb{R}^{|R| \times |R| \times |V|}$. These can be directly obtained by leveraging $\mE_h$ and $\mE_t$. For instance, the entry $\mA_R^{hh}[r_i, r_j, v_k]$, which counts how many times $v_k \in V$ is part of triplets involving both relation $r_i \in R$ and relation $r_j \in R$ while being the head entity in both, can be obtained by multiplying entries in $\mE_h$ as follows:
\begin{equation}
    \mA_R^{hh}[r_i, r_j, v_k] = \mE_h[v_k, r_i] * \mE_h[v_k, r_j].
\end{equation}
The equations for $\mA_R^{tt}$, $\mA_R^{ht}$, $\mA_R^{th}$ are obtained equivalently by substituting $\mE_h$ with $\mE_t$ appropriately and we refer the reader to \Cref{app:expressive-power} for explicit definitions.
These four intermediate adjacency relations are then stacked along the last dimensions, yielding a single $\mA_R \in \mathbb{R}^{|R| \times |R| \times |V| \times 4}$.
\Cref{fig:ourmethod} ((a), (b), (c) and (d)) contains an illustrative example of these four matrices. 
Finally, we remark that, although the shape $|R| \times |R| \times |V| \times 4$ of $\mA_R$ might look massive, it is in fact just an additional list of edge attributes and takes negligible additional space when relying on sparse matrix representations.

\subsection{Iterative Entity and Relation Embedding Updates}
\revision{Existing fully inductive models~\citep{galkin2023towards, lee2023ingram} sequentially perform message passing on the relation adjacency matrix to derive relation representations and subsequently on the entity adjacency matrix using the derived relation representations. As mentioned in \Cref{preliminary}, this sequential message passing does not align with the labeling trick and leads to inefficiency in relation prediction task. We propose a simultaneous refinement process through iterative updates which aligns well with labeling tricks of both relation and entity prediction task. With the more informative relation adjacency matrix $A_R$ proposed in the last part, this iterative embedding update scheme also makes full use of the entity information in $A_R$ to generate strictly more expressive triplet embeddings.} Specifically, we perform message passing updates on $\mA_R$, employing the entity representations as relation embeddings (since entities in $\mA_V$ correspond to relations in $\mA_R$). Subsequently, we proceed with message passing layers on $\mA_V$, utilizing the recently updated relation representations as the relation embeddings. This iterative process is repeated multiple times, ensuring a cohesive refinement of both relation and entity representations throughout the procedure.

In this subsection, we describe mathematically the iterative updates on the entity adjacency matrix and the relation adjacency matrix for the two tasks of interest, namely entity and relation predictions. \revision{\ourmethod is not jointly trained on both tasks, but the proposed framework can be adapted to either solve entity prediction or relation prediction tasks with different initial embeddings.} \revisionlog{The objective functions for optimizing the model for these tasks are in \Cref{loss_function}.}

Let $\mX^{(i)} \in \mathbb{R}^{|V| \times d}$ and $\mZ^{(i)} \in \mathbb{R}^{|R| \times d}$ denote the entity representations and the relation representations of dimension $d$ at any given layer $i$. Following previous methods~\citep{galkin2023towards,lee2023ingram}, we use \revision{NBFNet layers} with labeling tricks~\citep{zhang2021labeling} to obtain relative representations conditioned to the query of interest. Therefore, we will use subscripts to denote the conditioning set. For example, $\mX^{(i)}_{h,r}$ and $\mZ^{(i)}_{h,r}$ will denote the entity and relation representations conditioned on entity $h$ and relation $r$.
With a slight abuse of notation, we will then refer to the relative representation of entity $u$ conditioned on entity $h$ and relation $r$ as $\mX^{(i)}_{h,r}(u)$, and, similarly, $\mZ^{(i)}_{h,r}(r^\prime)$ will denote the relative representation of relation $r^\prime$ conditioned on entity $h$ and relation $r$.

\textbf{\revision{Embeddings for} Entity Prediction Tasks.}
For an entity prediction query $(h, r, ?)$, we leverage the labeling trick and initialize the embeddings of entities and relations to be conditioned on $h$ and $r$ as: 
\begin{align}
    \mX^{(0)}_{h,r}(u) = \revision{\text{INIT}_{V}(h, u)}, \qquad \mZ^{(0)}_{h,r}(r^\prime) = \revision{\text{INIT}_{R}(r, r^\prime)}
\end{align}

\revision{where \text{INIT} is the initialization function for embeddings. The initial embeddings for entity $h$ and for relation $r$ are set to all-one vectors, while all the other entities and relations receive initial zero vectors.}
Subsequently, we perform iterative updates as follows:
%
\begin{align}
    \mX_{h,r}^{(i)}(u) 
    &= \text{GNNLayer}_V^{(i)}(\mA_V, \mX_{h,r}^{(i-1)}, \mZ_{h,r}^{(i-1)}, u)  \nonumber \\
    &= \revision{\text{UP}_{V}^{(i)}\left(\mX_{h,r}^{(i-1)}(u), \text{AGG}_{V}^{(i)}\left(\text{MSG}_{V}^{(i)}(\mX_{h,r}^{(i-1)}(v), \mZ_{h,r}^{(i-1)}(r^\prime))|(u, r^\prime, v) \in \mA_V\right)\right)} \label{ee} \\
    \mZ_{h,r}^{(i)}(r^\prime) 
    &= \text{GNNLayer}_R^{(i)}(\mA_R, \mZ_{h,r}^{(i-1)}, \mX_{h,r}^{(i)}, r^\prime) \nonumber \\
    &= \revision{\text{UP}_{R}^{(i)}\left(\mZ_{h,r}^{(i-1)}(r^\prime), \text{AGG}_{R}^{(i)}\left(\text{MSG}_{R}^{(i)}(\mZ_{h,r}^{(i-1)}(r^{\prime\prime}), \mX_{h,r}^{(i)}(u))|(r^\prime,u,r^{\prime\prime})\in \mA_R\right)\right)} \label{er}
\end{align}
%
%
\revision{$\text{GNNLayer}_V^{(i)}$ and $\text{GNNLayer}_R^{(i)}$ are NBFNet layers where $\text{UP}^{(i)}$, $\text{AGG}^{(i)}$ and $\text{MSG}^{(i)}$ stand for update, aggregation, and message functions at the $i$-th layer, respectively. Following NBFNet, the message function is DistMult, the aggregation function is sum, and the update function is a multi-layer perceptron. The pseudo code is shown in \Cref{entity_code} in the Appendix.} The final entity embeddings are passed to a multi-layer perceptron for entity prediction. 

%

In knowledge graphs, every relation typically has a corresponding reverse relation. To handle the entity prediction query $(?, r, t)$, we can simply transform it into $(t, r^{-1}, ?)$ by utilizing the reverse relation and treat it as a tail prediction as above.

\textbf{\revision{Embeddings for} Relation Prediction Tasks.}
For a relation prediction query $(h, ?, t)$, we leverage the labeling trick and initialize the embeddings of entities and relations to be conditioned on $h$ and $t$ as follows:
\begin{align}
    \mX_{h,t}^{(0)}(u) = \revision{\text{INIT}(h, t, u)}, \qquad \mZ_{h,t}^{(0)}(r^\prime) = \revision{\mathbf{1}^{d}},
\end{align}
\revision{where $\mathbf{1}^{d}$ is an all-one vector of dimension $d$.} 
This means that the initial embedding for entity $h$ is set to all-one vectors, the initial embedding for entity $t$ is set to all-minus-one vectors, while all the other entities receive initial zero vectors. Furthermore, the relations are initialized all to all-one vectors.
Subsequently, we perform iterative embedding updates as follows.
Same as entity prediction, $\text{GNNLayer}_V^{(i)}$ and $\text{GNNLayer}_R^{(i)}$ are NBFNet layers:
\begin{align}
    \mZ_{h,t}^{(i)}(r) &= \text{GNNLayer}_R^{(i)}(\mA_R, \mZ_{h,t}^{(i-1)}, \mX_{h,t}^{(i-1)}, r) \nonumber\\
    &= \revision{\text{UP}_{R}^{(i)}\left(\mZ_{h,t}^{(i-1)}(r), \text{AGG}_{R}^{(i)}\left(\text{MSG}_{R}^{(i)}(\mZ_{h,t}^{(i-1)}(r^\prime), \mX_{h,t}^{(i-1)}(u))|(r,u,r^\prime)\in \mA_R\right)\right)} \label{rr}\\
    \mX_{h,t}^{(i)}(u) 
    &= \text{GNNLayer}_V^{(i)}(\mA_V, \mX_{h,t}^{(i-1)}, \mZ_{h,t}^{(i)}, u)  \nonumber\\
    &= \revision{\text{UP}_{V}^{(i)}\left(\mX_{h,t}^{(i-1)}(u), \text{AGG}_{V}^{(i)}\left(\text{MSG}_{V}^{(i)}(\mX_{h,t}^{(i-1)}(v), \mZ_{h,t}^{(i)}(r))|(u, r, v) \in \mA_V\right)\right)} \label{re}
\end{align}

The pseudo code is shown in \Cref{relation_code} in the Appendix. The final relation embeddings are passed to a multi-layer perceptron for the prediction. 

\subsection{{\revision{\ourmethod Properties}}}
\label{sec:expressivity}

We conduct a theoretical analysis to compare the expressiveness of \ourmethod with \revisionlog{existing relation-graph based fully-inductive models, namely ULTRA~\cite{galkin2023towards}, InGram~\cite{lee2023ingram} and DEq-InGram~\cite{gao2023double}.} More precisely, we aim to investigate the power of both methods to distinguish non-isomorphic triplets.
All proofs can be found in \Cref{app:expressive-power}. 
We begin by showing that \ourmethod is at least as expressive as \revisionlog{ULTRA and InGram}, since it can distinguish all non-isomorphic triplets that \revisionlog{they can distinguish. DEq-InGram is a variant of InGram by applying Monte Carlo sampling in inference~\cite{lee2023ingram}. Each sampling of DEq-InGram is actually a forward pass of InGram. So for simplicity in the proof we do not include DEq-InGram.}

\begin{restatable}[\ourmethod at least as powerful as ULTRA \revisionlog{and InGram}]{lemma}{asexpr}\label{theo:asexpr}
Any non-isomorphic triplet that can be distinguished by ULTRA \revisionlog{or InGram} can also be distinguished by \ourmethod\ 
 \revisionlog{with certain choices of hyperparameters and initial embeddings}.
\end{restatable}

We prove this by showing ULTRA \revisionlog{and InGram are actually special cases} of \ourmethod. That is, there exist choices of hyperparameters such that \ourmethod can precisely implement ULTRA \revisionlog{and InGram respectively} to reproduce the message passing process of them and can get the same triplet embeddings.
However, the contrary is not true: there exist cases where ULTRA \revisionlog{or InGram} cannot obtain the same triplet embeddings as \ourmethod, regardless of hyperparameter or weight choices. 


\begin{restatable}[\ourmethod can distinguish triplets ULTRA \revisionlog{and InGram} cannot]{lemma}{triplets}\label{theo:triplets}
There exist non-isomorphic triplets that can be distinguished by \ourmethod but that cannot be distinguished by ULTRA \revisionlog{and InGram}.
\end{restatable}

We prove this by constructing exemplary triplets that are clearly non-isomorphic, and then showing that they can be distinguished by \ourmethod but can not be distinguished by ULTRA \revisionlog{and InGram}. Finally, we can combine \Cref{theo:asexpr,theo:triplets} into \Cref{theo:moreexpr}. 

\begin{restatable}[\ourmethod is more expressive than ULTRA \revisionlog{and InGram}]{theorem}{moreexpr}\label{theo:moreexpr}
\ourmethod is strictly more expressive than ULTRA \revisionlog{and InGram} in distinguishing between non-isomorphic triplets in KGs.
\end{restatable}

We remark here that the additional expressive power comes from including which entities share two relations, at the cost of an additional dimension in the relation adjacency matrix. \revision{Prior relation graphs~\citep{galkin2023towards, lee2023ingram, gao2023double} consider all entities as isomorphic: the feature describing the entities shared by two relations is simply how many entities they have in common. 
It is not uncommon to have many pairs of relations with similar entity counts when these entities are non-isomorphic. Then the relations incorrectly tend to get similar or even same embedding. \ourmethod distinguish these relations by recording entities in the additional dimension in the relation graph. Through the proposed iterative updates, non-isomorphic entities get diverse embeddings and non-isomorphic relations also get diverse embeddings, thus the expressive power gets boosted.} In the following, we discuss the complexity of \ourmethod.

\textbf{Time Complexity.} Recall that the entity 
adjacency matrix $\mA_V \in \mathbb{R}^{|V| \times |V| \times |R|}$ has $\vert V \vert$ entities and $\vert R \vert$ relations. Denote by $\alpha$ the maximum number of unique relations one single entity connects to as the head or the tail in $\mA_V$. Then, the relation adjacency matrix $\mA_R \in \mathbb{R}^{|R| \times |R| \times |V| \times 4}$ has $\vert R \vert$ entities, $\vert V \vert$ relations and at most $4|V|\alpha^2$ edges. \revision{Assuming there are $L$ rounds of iterative updates} and embedding dimension is $d$, which we consider constant, the time complexity of \ourmethod is $O(|E|+|V|\revision{\alpha}^2)$ while the time complexity of ULTRA is $O(|E|+|V|+|R|^2)$ for each forward pass, which results in $O(|E|+|V|+|R|^2)$ for entity prediction and in $O((|E|+|V|+|R|^2)|R|)$ for relation prediction, as it needs to perform one forward pass for each relation in the relation set. 
We expand on this in \Cref{app:expressive-power}, where we show that, in practice, the number of edges in the relation graph is much smaller than $4|V|\alpha^2$, which therefore results in a complexity $\sim 10\times$ worse than ULTRA in entity prediction, and $\sim 20\times$ better than ULTRA in relation prediction.

\section{Experiments}
\label{experiment_section}
We perform a comprehensive set of experiments to answer the following questions:
\begin{enumerate*}[label=(\arabic*), leftmargin=*]
    \item How does \ourmethod compare to state-of-the-art \revision{fully inductive models} in inductive entity and relation prediction tasks and, in particular, does the increased expressiveness translate into better performance?
    \item Can the accuracy of \ourmethod increase with more data in the pre-training?
    \item \revisionlog{Can LLMs make inductive inferences based on the structural information in KGs? To be more specific,} are LLMs equivariant to permutations of relation and entity ids? And, can they do inductive tasks where the relation semantic is hidden and the model needs to leverage the \revisionlog{structural information in the input}?
\end{enumerate*}
\revision{In the following, we report our main results and refer to \Cref{appx:detailedresults} for additional experiments, including an ablation study on the importance of the proposed relation adjacency matrix and the iterative updates.}

\textbf{\ourmethod Implementation Details.}
\revision{We train one \ourmethod model for entity prediction and another \ourmethod model for relation prediction. For relation prediction,} \ourmethod updates relation and entity embedding for 3 rounds, while for entity prediction, for $5$ rounds. \revision{$\text{GNNLayer}_R$ and $\text{GNNLayer}_V$ in \Cref{ee,er,rr,re} are NBFNet layers with hidden dimension 32}. In pre-training, the model is trained for 10 epochs with 10000 steps per epoch with early stopping. In fine-tuning, the model is trained for 3 epochs with 1000 steps per epoch with early stopping. We use a batch size of 32, AdamW optimizer and learning rate of 0.0005. 

\subsection{Zero-Shot Inference and Fine-Tuning of \revision{Fully Inductive Models}}

\textbf{Datasets \& Evaluation.} We undertake a comprehensive evaluation across 57 distinct KGs coming from different domains, and split them into three categories: inductive entity and relation ($e$, $r$) datasets, inductive entity ($e$) datasets, and transductive datasets (\Cref{appx:datasets}). We follow the procedure prescribed by ULTRA~\citep{galkin2023towards} and pretrain on 3 datasets (WN18RR,
CoDEx-Medium, FB15k237). \revisionlog{We choose ULTRA as our baseline since it is the state-of-the-art double equivariant model that is capable of doing zero-shot fully inductive inference on the large KGs of our interests.} For entity prediction task, we report Mean Reciprocal Rank (MRR) and Hits@10 as the main performance metrics evaluated against the full entity set of the inference graph. For each triplet, we report the results of predicting both head and tail. Only in three datasets from \citet{lv2020dynamic} we report tail-only metrics, as in the baselines. For relation prediction, we report MRR and Hits@1 as the main performance metrics evaluated against the full relation set of the inference graph. We choose Hits@1 instead of Hits@10 because for some datasets the number of relations is less than 10. 




\textbf{Entity Prediction Results.}
We present the MRR and Hits@10 results across 57 KGs, along with the corresponding baseline (ULTRA) outcomes in \Cref{entity} \revision{and \Cref{entity_mrr_figure}}. Detailed per-dataset results, including standard deviations, are provided in \Cref{tab:app_zero_shot_ent,tab:app_fine_tune_ent}. In short, \ourmethod outperforms ULTRA by a significant margin ($\sim3\%$ average absolute improvement) in zero-shot scenarios, while demonstrating $\sim0.7\%$ average absolute improvement after fine-tuning in inference tasks. 
This demonstrates the importance of the additional expressiveness, which results in better performance. 

\revision{
\begin{figure}[t]
\centering
\includegraphics[width=1.0\textwidth]{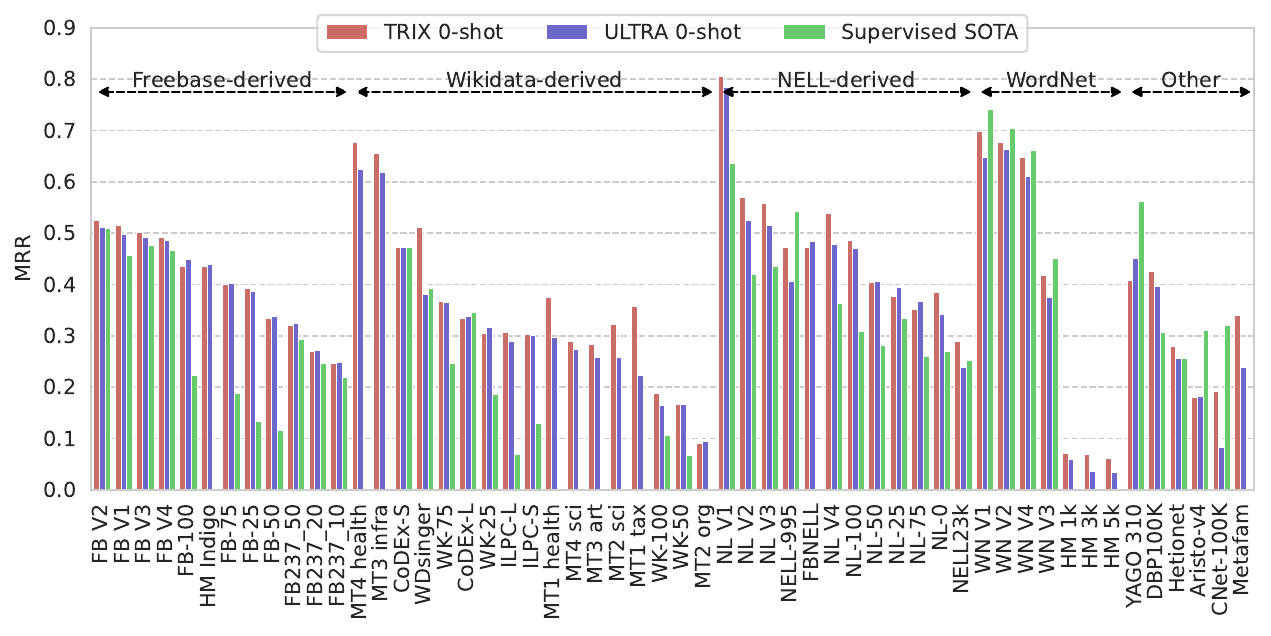}
\caption{ Zero-shot MRR (higher is better) in the entity prediction task. \ourmethod outperforms ULTRA on 34 datasets, while being comparable on 14 datasets and being outperformed on 6 datasets. It even outperforms supervised baselines on 30 out of 40 datasets.}
\label{entity_mrr_figure}
\vspace{-18pt}
\end{figure}
}

\begin{table*}[t]
    \centering
    \caption{Average entity prediction MRR and Hits@10 over 57 KGs from distinct domains.  The results over each of the 57 KGs are given in \Cref{appx:detailedresults}.}
    \label{entity}
\begin{adjustbox}{width=\textwidth}
    \begin{tabular}{lcccccccc||ccc}\toprule
    \multirow{3}{*}{\bf{Model}} & \multicolumn{2}{c}{\bf{Inductive} $e, r$} & \multicolumn{2}{c}{\bf{Inductive} $e$} & \multicolumn{2}{c}{\bf{Transductive}} & \multicolumn{2}{c}{\bf{Total Avg}} & \multicolumn{2}{c}{\bf{Pretraining}} \\  
 & \multicolumn{2}{c}{(23 graphs)} & \multicolumn{2}{c}{(18 graphs)} & \multicolumn{2}{c}{(13 graphs)} & \multicolumn{2}{c}{(54 graphs)} & \multicolumn{2}{c}{(3 graphs)} \\ \cmidrule(l){2-3} \cmidrule(l){4-5} \cmidrule(l){6-7} \cmidrule(l){8-9} \cmidrule(l){10-11}
 & \bf{MRR} & \bf{H@10} & \bf{MRR} & \bf{H@10} & \bf{MRR} & \bf{H@10} & \bf{MRR} & \bf{H@10} & \multicolumn{1}{c}{\bf{MRR}} & \bf{H@10} \\ 
    \midrule
    ULTRA zero-shot &0.345 &0.513 &0.431 &0.566 &0.312 &0.458 &0.366 &0.518 & N/A & N/A \\
    \ourmethod zero-shot & \textbf{0.368} & \textbf{0.540} & \textbf{0.455} & \textbf{0.592} & \textbf{0.339} & \textbf{0.500} & \textbf{0.390} & \textbf{0.548} & N/A & N/A \\ \midrule
    ULTRA fine-tuned & 0.397 & \textbf{0.556} & 0.442 & 0.582 &0.379 & 0.543 & 0.408 & 0.562 & 0.407 & \textbf{0.568} \\
    \ourmethod fine-tuned & \textbf{0.401} & \textbf{0.556} & \textbf{0.459} & \textbf{0.594} & \textbf{0.390} & \textbf{0.558} & \textbf{0.418} & \textbf{0.569} & \textbf{0.415} & 0.563 \\
    \bottomrule
    \end{tabular}
\end{adjustbox}
\end{table*}

\textbf{Relation Prediction Results of Graph Models.}
We present the average MRR and Hits@1 results across 57 KGs, alongside the results for ULTRA in \Cref{relation}  (Hits@10 is not chosen because some domains have too few relations). Detailed per-dataset results, including standard deviations, are available in  \Cref{tab:app_zero_shot_rel,tab:app_fine_tune_rel}. \ourmethod outperforms ULTRA significantly in zero-shot inference scenarios, exhibiting substantial advantages, with an average absolute improvement of 7.4\% in Hits@1. Moreover, after fine-tuning, \ourmethod shows an average absolute improvement of 4.7\% in Hits@1 in inference tasks. Notably, the zero-shot inference results obtained by \ourmethod even surpass those of ULTRA after finetuning. This underscores the effectiveness of \ourmethod in relation predictions.

\begin{table*}[t]
    \centering
    \caption{Average relation prediction MRR and hits@1 over 57 KGs from distinct domains.  The results over each of the 57 KGs are given in \Cref{appx:detailedresults}.}
    \label{relation}
\begin{adjustbox}{width=\textwidth}
    \begin{tabular}{lcccccccc||ccc}\toprule
    \multirow{3}{*}{\bf{Model}} & \multicolumn{2}{c}{\bf{Inductive} $e, r$} & \multicolumn{2}{c}{\bf{Inductive} $e$} & \multicolumn{2}{c}{\bf{Transductive}} & \multicolumn{2}{c}{\bf{Total Avg}} & \multicolumn{2}{c}{\bf{Pretraining}} \\  
 & \multicolumn{2}{c}{(23 graphs)} & \multicolumn{2}{c}{(18 graphs)} & \multicolumn{2}{c}{(13 graphs)} & \multicolumn{2}{c}{(54 graphs)} & \multicolumn{2}{c}{(3 graphs)} \\ \cmidrule(l){2-3} \cmidrule(l){4-5} \cmidrule(l){6-7} \cmidrule(l){8-9} \cmidrule(l){10-11}
 & \bf{MRR} & \bf{H@1} & \bf{MRR} & \bf{H@1} & \bf{MRR} & \bf{H@1} & \bf{MRR} & \bf{H@1} & \multicolumn{1}{c}{\bf{MRR}} & \bf{H@1} \\ 
    \midrule
    ULTRA zero-shot & 0.785 & 0.691 & 0.714 & 0.590 & 0.629 & 0.507 & 0.724 & 0.613 & N/A & N/A \\
    \ourmethod zero-shot & \textbf{0.842} & \textbf{0.770} & \textbf{0.756} & \textbf{0.611} & \textbf{0.752} & \textbf{0.647} & \textbf{0.792} & \textbf{0.687} & N/A & N/A \\ \midrule
    ULTRA fine-tuned  & 0.823 & 0.741 & 0.716 & 0.591 & 0.707 & 0.608 & 0.759 & 0.659 & 0.876 & \textbf{0.817} \\
    \ourmethod fine-tuned  & \textbf{0.850} & \textbf{0.785} & \textbf{0.759} & \textbf{0.615} & \textbf{0.785} & \textbf{0.693} & \textbf{0.804} & \textbf{0.706} & \textbf{0.879} & 0.797 \\
    \bottomrule
    \end{tabular}
\end{adjustbox}
\end{table*}

\begin{wrapfigure}[13]{}{0.6\textwidth}
\vspace{-3pt}
    \centering
    \begin{subfigure}[t]{0.48\linewidth}
        \centering
        \includegraphics[width=3.8cm]{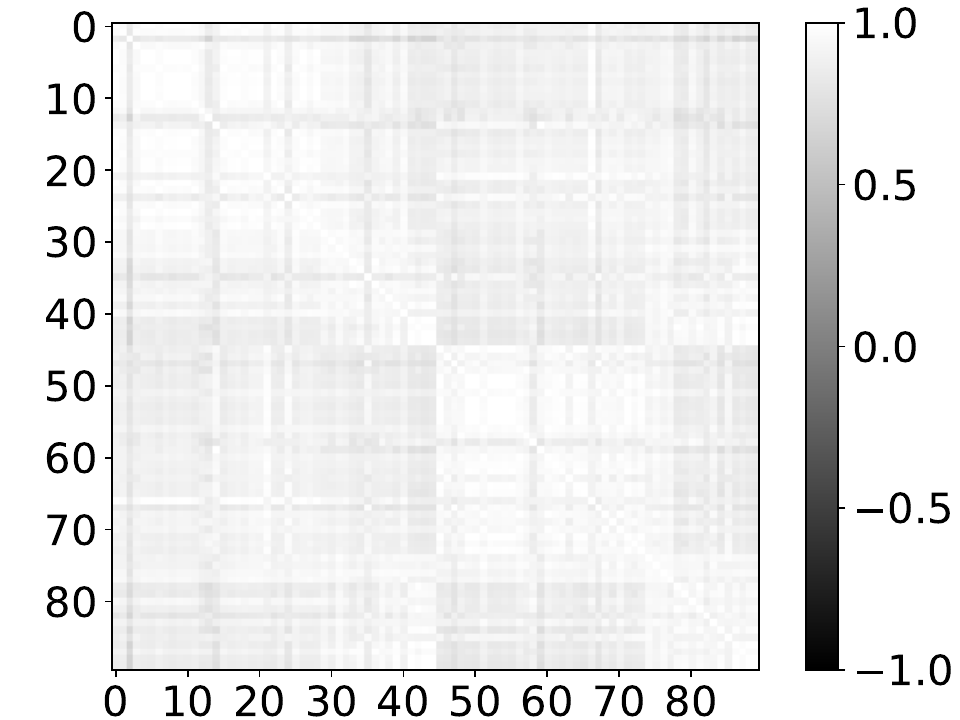}
        \caption{Single training domain.}
        \label{fig:heatmap1}
    \end{subfigure}
    \hfill
    \begin{subfigure}[t]{0.48\linewidth}
        \centering
        \includegraphics[width=3.8cm]{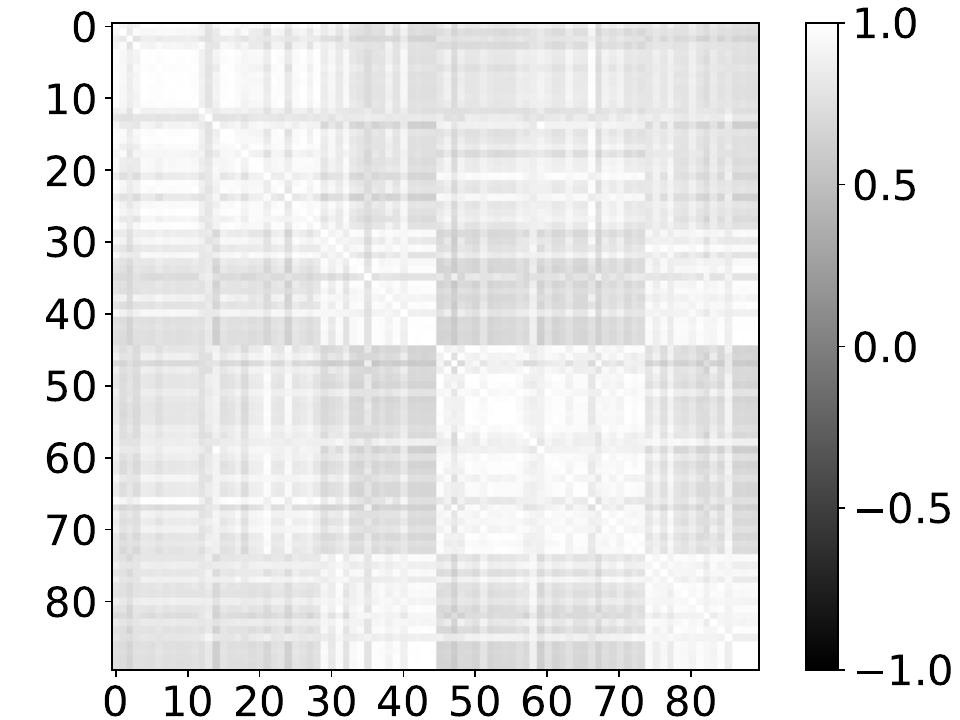}
        \caption{Four training domains.}
        \label{fig:heatmap2}
    \end{subfigure}
    \caption{Heatmaps of cosine similarities of relation embeddings in test with varying number of training domains. More domains shows stronger embedding differentiation.}
    \label{fig:heatmaps}
    \vspace{-15pt}
\end{wrapfigure}
\textbf{Testing Zero-shot Meta-Learning Ability of \ourmethod.}
In this experiment, we investigate the meta-learning capability of \ourmethod. 
That is, following \citet[pp.\ 4]{thrun1998learning}, we can define that a \revision{fully inductive model} is able to learn-to-learn when increasing the number of learning domains improves the inference performance on new domains. Here, we aim to understand the impact of including more domains in the zero-shot performance of \ourmethod. We use the WikiTopics dataset \cite{gao2023double}, which divides the Wikidata-5M~\cite{wang2021kepler} into 11 distinct, non-overlapping domains, resulting in 11 unique KGs, each containing different and non overlapping relation and entity sets. We then proceed by training \ourmethod on a randomly sampled increasing number of domains (until early-stopped) and evaluate their performance on the remaining domains. \Cref{entity_meta} show that as the number of domains in the training increases, the zero-shot inference capability of \ourmethod improves. Additionally, \Cref{fig:heatmap1,fig:heatmap2} illustrate the cosine similarities between the relation embeddings after training, averaged across different query relations, when training with one or four domains, respectively. By adding domains in the training mixture, the model can see more invariances and the relation embeddings get more distinct and expressive, which then translates into more accurate prediction in unseen domains.


\begin{table*}[t]
\vspace{-5pt}
\caption{\ourmethod Entity Prediction Hits@10 on WiKiTopics with varying pre-training domains.}
\centering
\label{entity_meta}
\begin{adjustbox}{width=\textwidth}
\begin{tabular}{l|c|cccccccc}
\toprule
\# pre-train domains & Average & Art & Award & Edu & Health & Infra & Sci & Sport & Tax \\
\midrule
1 domain & 0.413 & 0.380 & 0.428 & 0.263 & 0.601 & 0.556 & 0.369 & \textbf{0.385} & 0.324 \\
2 domains & 0.459 & \textbf{0.439} & 0.462 & 0.282 & 0.700 & 0.656 & 0.423 & 0.366 & 0.358 \\
3 domains & 0.458 & 0.432 & \textbf{0.472} & \textbf{0.300} & 0.623 & 0.679 & 0.431 & 0.376 & 0.361 \\
4 domains & \textbf{0.483} & 0.432 & 0.464 & 0.294 & \textbf{0.744} & \textbf{0.724} & \textbf{0.448} & 0.382 & \textbf{0.396} \\
\bottomrule
\end{tabular}
\end{adjustbox}
\vspace{-10pt}
\end{table*}

\subsection{Relation and Entity Prediction with LLMs}
To evaluate whether long-context LLMs can make zero-shot relation \revision{and entity} predictions based on the structural patterns of the graph, we design three different tasks
\revision{containing both relation prediction and entity prediction.} In these tasks, all triplets (head entity, relation, tail entity) of the KG are provided in the prompt. The tasks differ in how entities and relations are represented. \revisionlog{The goal of these experiments is to underscore the shortcomings of LLMs in capturing the structural patterns, which is due to being sensitive to permutations of IDs and relying on semantic information.}

\revisionlog{\textbf{Experiment Setup.} Since all the triplets of KGs are provided in the prompt (approximately 700,000 tokens per query), we selected the Gemini Pro models (Gemini-1.5-flash and Gemini-1.5-pro~\citep{geminiteam2024gemini}) as the baseline because they offer the largest token size (1,048,576 tokens for Gemini-1.5-flash), whereas other models, such as GPT-4 (32,000 tokens) and Llama 3.1-70B Instruct (128,000 tokens), have limited token capacities. Due to the cost of the Gemini Pro model, we evaluated on 30 samples from the CoDEx-S dataset~\cite{safavi2020codex}. The prompts used followed previous works~\cite{yao2023exploring, shu2024knowledge}. We report in the following relation prediction results, while entity prediction results and further experimental details, including prompts, can be found in \Cref{appx:LLMexp}.}

\begin{wrapfigure}{r}{0.5\textwidth}
\vspace{-2em}
\centering
\includegraphics[width=0.5\textwidth]{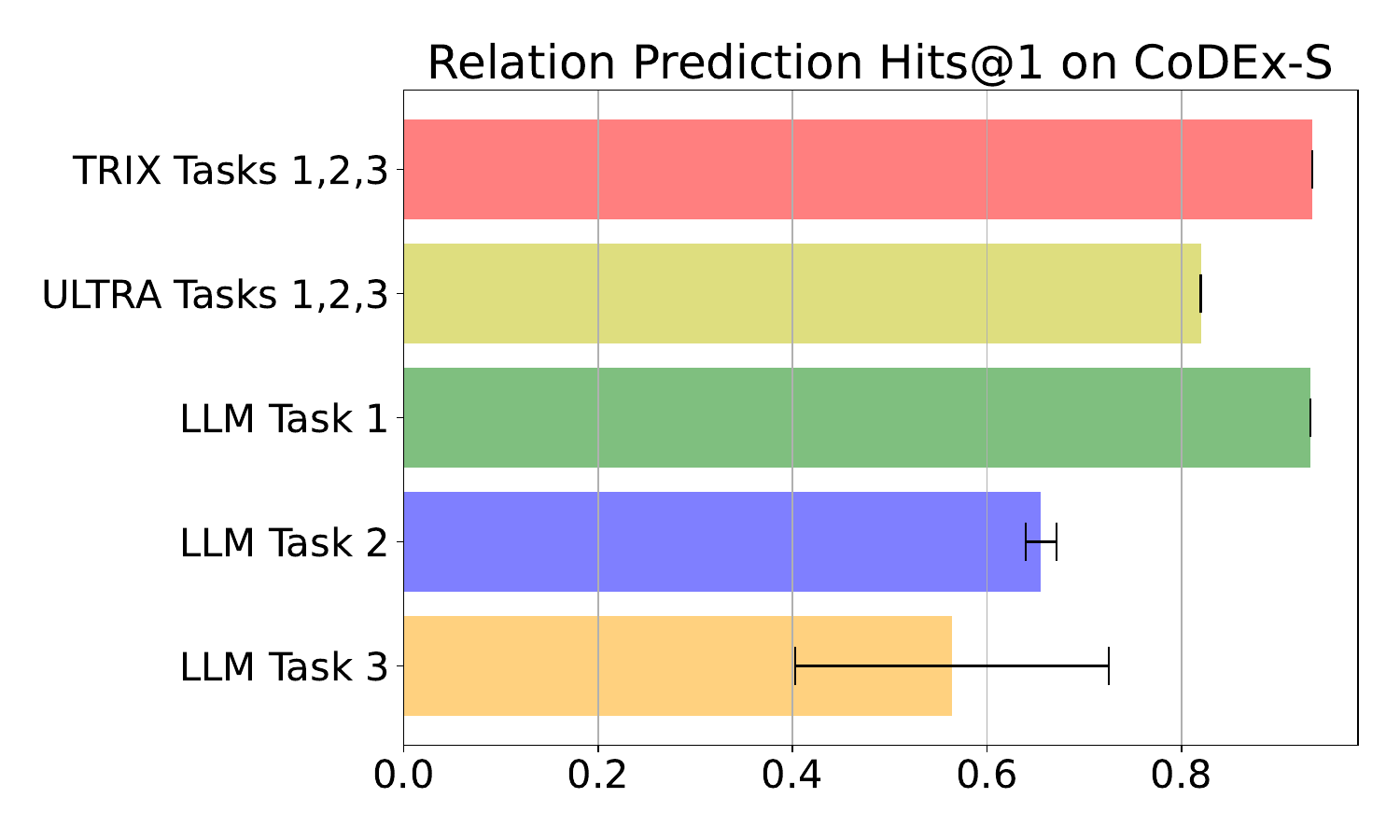}
\caption{Relation prediction Hits@1 of 
Gemini-1.5-Pro on three tasks against \revision{fully inductive models}. Double-equivariant graph models always give consistent performance but LLM performance changes drastically with metasyntactic tokens (Task 2) and ID permutations (Task 3).}
\vspace{-22pt}
\label{fig:llmplot}
\end{wrapfigure}
%
%

\textbf{Task 1: In-domain LLM predictions.}
The entities and relations are expressed by \revision{their names} in natural language, which means that the LLM is queried in the way it was trained on. This task evaluates the in-domain prediction capacity of long-context LLMs. \Cref{fig:llmplot} (and \Cref{appx:T1}) shows that the LLM handles the in-domain relation prediction nearly as well as \ourmethod and better than ULTRA. 

\textbf{Task 2: Out-of-domain LLM predictions.}
The neighbor entities of the head entity and the relations that connect the head entity with its neighbors are replaced with metasyntactic words~\cite{metasyntactic} like \texttt{foo}, \texttt{bar} and \texttt{baz} across the whole KG to simulate information from a new domain (with new entities and relations) \revision{while other entities and relations are still expressed by their names in natural language}. This task tests whether long-context LLMs reason on the KG's structural information for the inductive predictions of the metasyntactic words, or if they just simply rely on its pre-trained semantic information. \Cref{fig:llmplot} (details in \Cref{appx:T2}) shows that with metasyntactic words, Gemini-1.5-pro fails to make consistent accurate predictions. Hence, the good performance in Task 1 likely came from its in-distribution syntax, which means that even SOTA long-context LLMs struggle making relation predictions in new KGs domains using only relational information.

\textbf{Task 3: Double-equivariant LLM predictions.}
All the entities and relations are expressed with IDs (e.g. ``entity 64'', ``relation 22'') as in \citet{shu2024knowledge}. The mappings from entities and relations to their IDs are given to the LLM and the IDs are shuffled in each test run. This task tests whether long-context LLMs are (double) equivariant to entity and relation ID permutations.
\Cref{fig:llmplot} (details in \Cref{appx:T3}) shows that the performance of a long-context LLM changes drastically with the permutation of relation and entity IDs. Only for $38.5\%$ of all queries the model makes consistent predictions across 3 runs, which indicates LLM is quite sensitive to the input permutations of the KG. With the whole graph sent as an edge list, the task resembles long text understanding and retrieval tests akin to ``a needle in a haystack''~\cite{babilong,zeroscrolls,longbench} where LLMs show increasingly better performance on context lengths up to 128k tokens. However, our results indicate that such (often synthetic) benchmarks might be overestimating real long-context reasoning capabilities of LLMs (let alone adding a simple relation prediction task on top of the input context).

\section{Conclusion}
In this paper we considered the fully inductive link prediction task in KGs. We identified the open challenges in existing fully inductive models, and proposed \ourmethod, a novel architecture designed to improve expressiveness and support efficient relation prediction tasks. Through comprehensive experiments spanning 57 diverse KGs datasets, we demonstrate that increased expressiveness translates into better performance. Additionally, our experimental study sheds light on the limitations of LLMs in exploiting graph information in new domains for entity and relation prediction tasks.


\newpage
\clearpage
\section*{Acknowledgments}
This work was funded in part by the National Science Foundation (NSF) awards, CCF-1918483, CAREER IIS-1943364 and CNS-2212160, Amazon Research Award, AnalytiXIN, and the Wabash Heartland Innovation Network (WHIN), Ford, NVidia, CISCO, and Amazon. Computing infrastructure was supported in part by CNS-1925001 (CloudBank). This work was supported in part by AMD under the AMD HPC Fund program. 

\bibliographystyle{unsrtnat}
\bibliography{ref}

\newpage
\appendix
\section{Pseudo Codes of Iterative Embedding Updates}
\Cref{entity_code} and \Cref{relation_code} show the iterative embedding updates in entity prediction task and relation prediction task respectively.

\begin{minipage}{\textwidth}
\begin{algorithm}[H]
\small
\caption{\ourmethod embedding updates for the entity prediction task}
\label{entity_code}
\begin{algorithmic}
\Require Query ($h$, $r$, ?); relation adjacency matrix $\mA_R$; entity adjacency matrix $\mA_V$; number of updates $L$
\Ensure Final entity embedding $\mX_{h,r}^{(L)}$
\State $\mX_{h,r}^{(0)} = \text{INIT}_{V}(h)$ \Comment{Label $h$ with all-ones vector and the rest with all-zeros}
\State $\mZ_{h,r}^{(0)} = \text{INIT}_{R}(r)$ \Comment{Label $r$ with all-ones vector and the rest with all-zeros}
\For{$i \gets 1$ to $L$}
\For{$u \in V$}
\State $\mX_{h,r}^{(i)}(u) = \text{UP}_{V}^{(i)}\left(\mX_{h,r}^{(i-1)}(u), \text{AGG}_{V}^{(i)}\left(\text{MSG}_{V}^{(i)}(\mX_{h,r}^{(i-1)}(v), \mZ_{h,r}^{(i-1)}(r^\prime))|(u, r^\prime, v) \in \mA_V\right)\right)$
\EndFor
\For{$r^\prime \in R$}
\State $\mZ_{h,r}^{(i)}(r^\prime) = \text{UP}_{R}^{(i)}\left(\mZ_{h,r}^{(i-1)}(r^\prime), \text{AGG}_{R}^{(i)}\left(\text{MSG}_{R}^{(i)}(\mZ_{h,r}^{(i-1)}(r^{\prime\prime}), \mX_{h,r}^{(i)}(u))|(r^\prime,u,r^{\prime\prime})\in \mA_R\right)\right)$
\EndFor
\EndFor
\end{algorithmic}    
\end{algorithm}
\end{minipage}

\begin{minipage}{\textwidth}
\begin{algorithm}[H]
\small
\caption{\ourmethod embedding updates for the relation prediction task}
\label{relation_code}
\begin{algorithmic}
\Require Query ($h$, ?, $t$); relation adjacency matrix $\mA_R$; entity adjacency matrix $\mA_V$; number of updates $L$
\Ensure Final relation embedding $\mZ_{h,t}^{(L)}$
\State $\mZ_{h,t}^{(0)} = \mathbf{1}^{|R|\times d}$
\State $\mX_{h,t}^{(0)} = \text{INIT}(h, t)$ \Comment{Label $h$ with all-ones vector, $t$ with all-negative-ones and the rest with all-zeros}
\For{$i \gets 1$ to $L$}
\For{$r \in R$}
\State $\mZ_{h,t}^{(i)}(r) = \text{UP}_{R}^{(i)}\left(\mZ_{h,t}^{(i-1)}(r), \text{AGG}_{R}^{(i)}\left(\text{MSG}_{R}^{(i)}(\mZ_{h,t}^{(i-1)}(r^\prime), \mX_{h,t}^{(i-1)}(u))|(r,u,r^\prime)\in \mA_R\right)\right)$
\EndFor
\For{$u \in V$}
\State $\mX_{h,t}^{(i)}(u) = \text{UP}_{V}^{(i)}\left(\mX_{h,t}^{(i-1)}(u), \text{AGG}_{V}^{(i)}\left(\text{MSG}_{V}^{(i)}(\mX_{h,t}^{(i-1)}(v), \mZ_{h,t}^{(i)}(r))|(u, r, v) \in \mA_V\right)\right)$
\EndFor
\EndFor
\end{algorithmic}    
\end{algorithm}
\end{minipage}

\section{Expressive Power}\label{app:expressive-power}
\label{proof_section}

\begin{wrapfigure}[14]{}{0.6\textwidth}
\vspace{-10pt}
    \centering
    \begin{subfigure}[t]{0.48\linewidth}
        \centering
        \includegraphics[width=3.8cm]{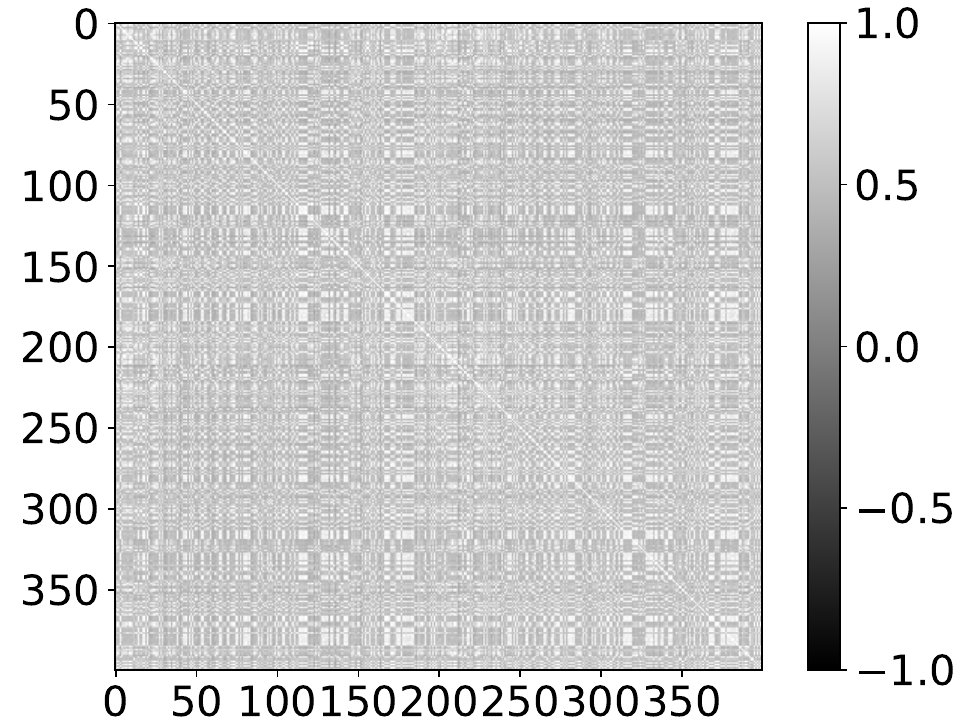}
        \caption{\ourmethod.}
        \label{fig:heatmaptrix}
    \end{subfigure}
    \hfill
    \begin{subfigure}[t]{0.48\linewidth}
        \centering
        \includegraphics[width=3.8cm]{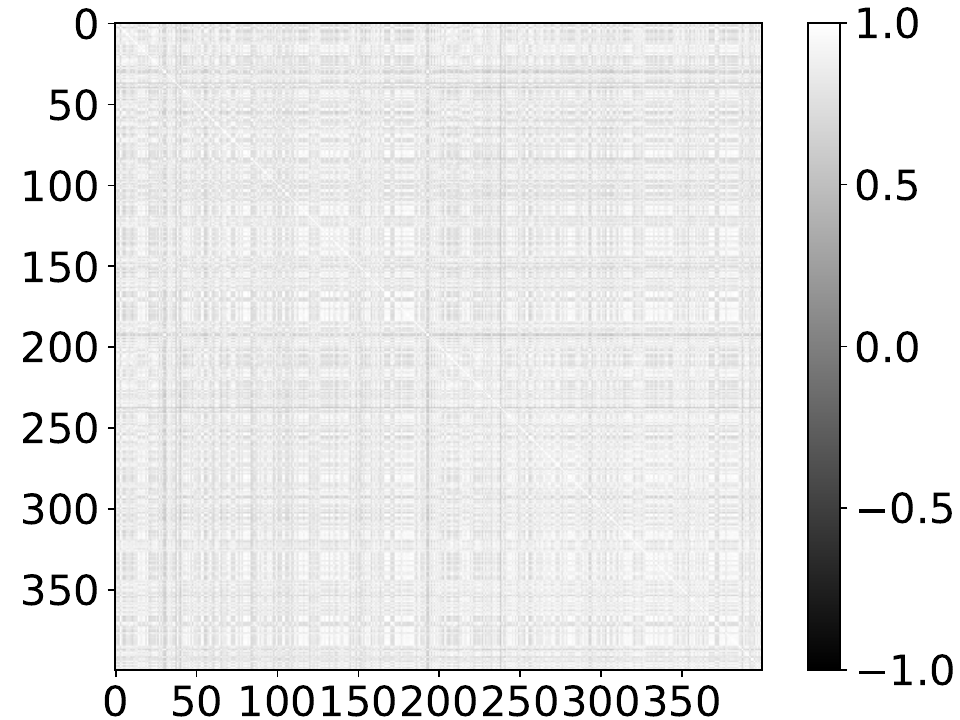}
        \caption{ULTRA.}
        \label{fig:heatmapultra}
    \end{subfigure}
    \caption{\revisionlog{Heatmaps of cosine similarities of relation embeddings on NELL995 with \ourmethod and ULTRA. \ourmethod gets more diverse relation embeddings than that of ULTRA.}}
    \label{fig:heatmapscompare}
    \vspace{-15pt}
\end{wrapfigure}

The expressive power of a GNN refers to its ability of distinguishing non-isomorphic graphs. Previous works~\cite{morris2019weisfeiler} have shown that the expressive power of GNN is bounded by Weisfeiler–Leman tests. Recent works~\cite{barcelo2022weisfeiler} also extend Weisfeiler–Leman tests on relational graphs. The expressive power in link prediction is limited due to the automorphic node problem~\cite{chamberlain2022graph}. Labeling trick~\cite{zhang2021labeling, zhu2021neural} breaks the symmetry of entity embedding in link prediction by assigning each node a unique feature vector based on its structural properties. 

\revisionlog{\Cref{sec:expressivity} shows that \ourmethod is more expressive than existing relation-graph based double-equivariant models. Before going into the details, we discuss the implications of this expressivity increase, such that the relation embeddings that become more distinguishable. \Cref{fig:heatmapscompare} shows the heatmaps of cosine similarity of relation embeddings of ULTRA and \ourmethod on the NELL995 dataset. The heatmap of \ourmethod shows more shade ranges than that of ULTRA, which indicates that the relation embeddings of \ourmethod are more distinct.
}

In the following part we discuss the expressive power of \ourmethod. 
We start by formally introducing the four intermediate relation adjacency matrices. Specifically, recall that the entity
adjacency matrix $\mA_V \in \mathbb{R}^{|V| \times |V| \times |R|}$ is obtained by stacking four relation adjacency matrices capturing the four roles $\mA_R^{hh}$ (head-head), $\mA_R^{tt}$ (tail-tail), $\mA_R^{ht}$ (head-tail), and $\mA_R^{th}$ (tail-head). These can be directly obtained by leveraging $\mE_h$ and $\mE_t$, that is:
\begin{align*}
\mA_R^{hh}[r_i, r_j, v_k] &= \mE_h[v_k, r_i] * \mE_h[v_k, r_j], \qquad \mA_R^{tt}[r_i, r_j, v_k] = \mE_t[v_k, r_i] * \mE_t[v_k, r_j],\\
\mA_R^{ht}[r_i, r_j, v_k] &= \mE_h[v_k, r_i] * \mE_t[v_k, r_j], \qquad
\mA_R^{th}[r_i, r_j, v_k] = \mE_t[v_k, r_i] * \mE_h[v_k, r_j].
\end{align*}

Next we will discuss the details of the proof.

\asexpr*

\begin{proof} 
\revision{
In this proof we consider the task of entity prediction. 

\revisionlog{\textbf{\ourmethod is at least as expressive as ULTRA.}} Suppose ULTRA first updates relation embedding with a $k_1$-layer GNN and then updates entity embedding with another $k_2$-layer GNN. We will show that there is a \ourmethod with $(k_1+k_2)$ rounds of iterative updates that will get the same embeddings as ULTRA. \ourmethod uses the same initial relation entity embedding as ULTRA and uses the all-one matrix as initial entity embedding. \revisionlog{To be consistent with ULTRA, \ourmethod uses NBFNet as the GNNLayer.}

For the first $k_1$ rounds, ULTRA only updates relation embedding with message passing on relation graphs. ULTRA's message passing function on relation graph is
\begin{align*}
    \mZ_{h,r}^{(i)}(r^\prime) &= \revision{\text{UP}_{R, \text{ULTRA}}^{(i)}\left(\mZ_{h,r}^{(i-1)}(r^\prime), \text{AGG}_{R, \text{ULTRA}}^{(i)}\left(\text{MSG}_{R, \text{ULTRA}}^{(i)}(\mZ_{h,r}^{(i-1)}(r^{\prime\prime})|(r^\prime,r^{\prime\prime})\in \mA^{\text{ULTRA}}_R\right)\right)}
\end{align*}

For \ourmethod, we choose the message passing functions on relation graph as follows: 
\begin{align*}
    \mX_{h,r}^{(i)}(u) &= \text{UP}_{V, \text{\ourmethod}}^{(i)}\left(\mX_{h,r}^{(i-1)}(u), \right. \\
    &\left.\text{AGG}_{V, \text{\ourmethod}}^{(i)}\left(\text{MSG}_{V, \text{\ourmethod}}^{(i)}(\mX_{h,r}^{(i-1)}(v), \mZ_{h,r}^{(i-1)}(r^\prime))|(u, r^\prime, v) \in \mA^{\text{\ourmethod}}_V\right)\right) \\
    &= \mX_{h,r}^{(i-1)}(u) = \mathbf{1}^{d}\\
    \mZ_{h,r}^{(i)}(r^\prime) &= \text{UP}_{R, \text{\ourmethod}}^{(i)}\left(\mZ_{h,r}^{(i-1)}(r^\prime),\right.\\ &\left.\text{AGG}_{R, \text{\ourmethod}}^{(i)}\left(\text{MSG}_{R, \text{\ourmethod}}^{(i)}(\mZ_{h,r}^{(i-1)}(r^{\prime\prime}), \mX_{h,r}^{(i)}(u))|(r^\prime,u,r^{\prime\prime})\in \mA^{\text{\ourmethod}}_R\right)\right)
\end{align*}

Both ULTRA and \ourmethod use the non-parametric DistMult as the message function. If $\mX_{h,r}^{(i)}$ is an all-one matrix, we have
\begin{align*}
&\text{MSG}_{R, \text{ULTRA}}^{(i)}(\mZ_{h,r}^{(i-1)}(r^{\prime\prime})|(r^\prime,r^{\prime\prime})\in \mA^{\text{ULTRA}}_R)\\
=&\text{MSG}_{R, \text{\ourmethod}}^{(i)}(\mZ_{h,r}^{(i-1)}(r^{\prime\prime}), \mX_{h,r}^{(i)}(u))|(r^\prime,u,r^{\prime\prime})\in \mA^{\text{\ourmethod}}_R)
\end{align*}

Both ULTRA and \ourmethod use sum as the aggregation function. If $\text{UP}_{R, \text{ULTRA}}^{(i)}$ and $\text{UP}_{R, \text{\ourmethod}}^{(i)}$ always have the same parameter, it will end up with the same $\mZ_{h,r}^{(k_1)}$.

For the last $k_2$ layers, ULTRA only updates entity embedding. ULTRA marks the head entity with the query relation embedding. ULTRA's message passing function on the original knowledge graph is:
\begin{align*}
\mX_{h,r}^{(i)}(u) &= \text{UP}_{V, \text{ULTRA}}^{(i)}\left(\mX_{h,r}^{(i-1)}(u), \right. \\
&\left. \text{AGG}_{V, \text{ULTRA}}^{(i)}\left(\text{MSG}_{V, \text{ULTRA}}^{(i)}(\mX_{h,r}^{(i-1)}(v), \mZ_{h,r}^{(k_1)}(r^\prime))|(u, r^\prime, v) \in \mA^{\text{ULTRA}}_V\right)\right)
\end{align*}

For \ourmethod, it also marks the head entity with the query relation embedding as ULTRA does. We choose the message passing functions on relation graph as follows: 
\begin{align*}
    \mX_{h,r}^{(i+k_1)}(u) &= \text{UP}_{V, \text{\ourmethod}}^{(i+k_1)}\left(\mX_{h,r}^{(i+k_1-1)}(u),\right.\\
    &\left.\text{AGG}_{V, \text{\ourmethod}}^{(i+k_1)}\left(\text{MSG}_{V, \text{\ourmethod}}^{(i+k_1)}(\mX_{h,r}^{(i+k_1-1)}(v), \mZ_{h,r}^{(i+k_1-1)}(r^\prime))|(u, r^\prime, v) \in \mA^{\text{\ourmethod}}_V\right)\right) \\
    \mZ_{h,r}^{(i+k_1)}(r^\prime) &= \text{UP}_{R, \text{\ourmethod}}^{(i+k_1)}\left(\mZ_{h,r}^{(i+k_1-1)}(r^\prime),\right.\\ 
    &\left.\text{AGG}_{R, \text{\ourmethod}}^{(i+k_1)}\left(\text{MSG}_{R, \text{\ourmethod}}^{(i+k_1)}(\mZ_{h,r}^{(i+k_1-1)}(r^{\prime\prime}), \mX_{h,r}^{(i+k_1)}(u))|(r^\prime,u,r^{\prime\prime})\in \mA^{\text{\ourmethod}}_R\right)\right) \\
    &= \mZ_{h,r}^{(i+k_1-1)}(r^\prime) = \mZ_{h,r}^{(k_1)}
\end{align*}

If $\text{UP}_{V, \text{ULTRA}}^{(i)}$ and $\text{UP}_{V, \text{\ourmethod}}^{(i+k_1)}$ always have the same parameter, the relation and entity embedding from ULTRA and \ourmethod are exactly the same. Since the embeddings are the same, any non-isomorphic triplet that can be distinguished by ULTRA can also be distinguished by \ourmethod.
}

\revisionlog{\textbf{\ourmethod is at least as expressive as InGram.} Suppose InGram first updates relation embedding with a $k_1$-layer GNN and then updates entity embedding with another $k_2$-layer GNN. We will show that there is a \ourmethod with $(k_1+k_2)$ rounds of iterative updates that will get the same embeddings as InGram. \ourmethod uses the same random initial relation entity embedding as InGram and uses the all-one matrix as initial entity embedding. To be consistent with InGram, \ourmethod uses an extension of GATv2 as the GNNLayer.


For the first $k_1$ rounds, InGram only updates relation embedding with message passing on relation graphs. InGram's message passing function on relation graph is
\begin{align*}
    \mZ_{h,r}^{(i)}(r^\prime) &= \text{UP}_{R, \text{InGram}}^{(i)}\left(\mZ_{h,r}^{(i-1)}(r^\prime), \text{AGG}_{R, \text{InGram}}^{(i)}\left(\text{MSG}_{R, \text{InGram}}^{(i)}(\mZ_{h,r}^{(i-1)}(r^{\prime\prime})|(r^\prime,r^{\prime\prime})\in \mA^{\text{InGram}}_R\right)\right)
\end{align*}

For \ourmethod, we choose the message passing functions on relation graph as follows. Here with a slight abuse of notation, $\mA^{\text{\ourmethod}}_{R,hh}$ and $\mA^{\text{\ourmethod}}_{R,tt}$
are the relation adjacency matrices for the head-head and tail-tail connections respectively.
\begin{align*}
    \mX_{h,r}^{(i)}(u) &= \text{UP}_{V, \text{\ourmethod}}^{(i)}\left(\mX_{h,r}^{(i-1)}(u), \right. \\
    &\left.\text{AGG}_{V, \text{\ourmethod}}^{(i)}\left(\text{MSG}_{V, \text{\ourmethod}}^{(i)}(\mX_{h,r}^{(i-1)}(v), \mZ_{h,r}^{(i-1)}(r^\prime))|(u, r^\prime, v) \in \mA^{\text{\ourmethod}}_V\right)\right) \\
    &= \mX_{h,r}^{(i-1)}(u) = \mathbf{1}^{d}\\
    \mZ_{h,r}^{(i)}(r^\prime) &= \text{UP}_{R, \text{\ourmethod}}^{(i)}\left(\mZ_{h,r}^{(i-1)}(r^\prime),\right.\\ &\left.\text{AGG}_{R, \text{\ourmethod}}^{(i)}\left(\text{MSG}_{R, \text{\ourmethod}}^{(i)}(\mZ_{h,r}^{(i-1)}(r^{\prime\prime}), \mX_{h,r}^{(i)}(u))|(r^\prime,u,r^{\prime\prime})\in \left(\mA^{\text{\ourmethod}}_{R,hh} \cup \mA^{\text{\ourmethod}}_{R,tt}\right) \right)\right)
\end{align*}

\ourmethod picks the same message function as InGram. If $\mX_{h,r}^{(i)}$ is an all-one matrix, we have
\begin{align*}
&\text{MSG}_{R, \text{InGram}}^{(i)}(\mZ_{h,r}^{(i-1)}(r^{\prime\prime})|(r^\prime,r^{\prime\prime})\in \mA^{\text{InGram}}_R)\\
=&\text{MSG}_{R, \text{\ourmethod}}^{(i)}(\mZ_{h,r}^{(i-1)}(r^{\prime\prime}), \mX_{h,r}^{(i)}(u))|(r^\prime,u,r^{\prime\prime})\in \left(\mA^{\text{\ourmethod}}_{R,hh} \cup \mA^{\text{\ourmethod}}_{R,tt}\right))
\end{align*}

Both InGram and \ourmethod use sum as the aggregation function. If $\text{UP}_{R, \text{InGram}}^{(i)}$ and $\text{UP}_{R, \text{\ourmethod}}^{(i)}$ always have the same parameter, it will end up with the same $\mZ_{h,r}^{(k_1)}$.

For the last $k_2$ layers, InGram only updates entity embedding. InGram's message passing function on the original knowledge graph is:
\begin{align*}
\mX_{h,r}^{(i)}(u) &= \text{UP}_{V, \text{InGram}}^{(i)}\left(\mX_{h,r}^{(i-1)}(u), \right. \\
&\left. \text{AGG}_{V, \text{InGram}}^{(i)}\left(\text{MSG}_{V, \text{InGram}}^{(i)}(\mX_{h,r}^{(i-1)}(v), \mZ_{h,r}^{(k_1)}(r^\prime))|(u, r^\prime, v) \in \mA^{\text{InGram}}_V\right)\right)
\end{align*}

For \ourmethod, we choose the message passing functions on relation graph as follows and change the entity embedding to the same initial entity embedding as InGram: 
\begin{align*}
    \mX_{h,r}^{(i+k_1)}(u) &= \text{UP}_{V, \text{\ourmethod}}^{(i+k_1)}\left(\mX_{h,r}^{(i+k_1-1)}(u),\right.\\
    &\left.\text{AGG}_{V, \text{\ourmethod}}^{(i+k_1)}\left(\text{MSG}_{V, \text{\ourmethod}}^{(i+k_1)}(\mX_{h,r}^{(i+k_1-1)}(v), \mZ_{h,r}^{(i+k_1-1)}(r^\prime))|(u, r^\prime, v) \in \mA^{\text{\ourmethod}}_V\right)\right) \\
    \mZ_{h,r}^{(i+k_1)}(r^\prime) &= \text{UP}_{R, \text{\ourmethod}}^{(i+k_1)}\left(\mZ_{h,r}^{(i+k_1-1)}(r^\prime),\right.\\ 
    &\left.\text{AGG}_{R, \text{\ourmethod}}^{(i+k_1)}\left(\text{MSG}_{R, \text{\ourmethod}}^{(i+k_1)}(\mZ_{h,r}^{(i+k_1-1)}(r^{\prime\prime}), \mX_{h,r}^{(i+k_1)}(u))|(r^\prime,u,r^{\prime\prime})\in \mA^{\text{\ourmethod}}_R\right)\right) \\
    &= \mZ_{h,r}^{(i+k_1-1)}(r^\prime) = \mZ_{h,r}^{(k_1)}
\end{align*}

If $\text{UP}_{V, \text{InGram}}^{(i)}$ and $\text{UP}_{V, \text{\ourmethod}}^{(i+k_1)}$ always have the same parameter, the relation and entity embedding from InGram and \ourmethod are exactly the same. Since the embeddings are the same, any non-isomorphic triplet that can be distinguished by InGram can also be distinguished by \ourmethod.
}
\end{proof} 

\triplets*
\begin{proof}
\revisionlog{\textbf{\ourmethod is more expressive than ULTRA.}} Consider the two triplets ($v_{21}$, $r_3$, $v_{24}$) and ($v_{21}$, $r_3$, $v_{25}$) in \Cref{fig:example}, which are non-isomorphic because $r_1 \neq r_2$, and assume they are not observable (i.e., they do not exist in the input graph). Assume, however, we want to predict the existence of these two, with the first one having label 1 (exists) and the second one having label 0 (does not exist). In practice, this means that the first one represents a missing link in an incomplete KG, while the other does not exist and should not be predicted as missing. The proof shows the impossibility of ULTRA to distinguish these two and therefore to make different predictions for the two.

In the relation graph of ULTRA, $r_1$ and $r_2$ are \revisionlog{automorphic}. This implies that ULTRA will give $r_1$ and $r_2$ the same relation embedding. Since this embedding is then used by ULTRA to obtain the entity embeddings, ULTRA will assign $v_{24}$ and $v_{25}$ the same entity embedding. Therefore, ULTRA will assign ($v_{21}$, $r_3$, $v_{24}$) and ($v_{21}$, $r_3$, $v_{25}$) the same representation, and therefore it will necessarily make the same prediction for these two (predict that they both exist or that they both do not exist, even though the ground truth tell us only one exist).

On the contrary, in the relation graph of \ourmethod, $r_1$ and $r_2$ are not \revisionlog{automorphic}. This implies that \ourmethod can give $r_1$ and $r_2$ different relation embeddings, and consequently $v_{24}$ and $v_{25}$ can get different entity embeddings. Therefore, \ourmethod can assign ($v_{21}$, $r_3$, $v_{24}$) and ($v_{21}$, $r_3$, $v_{25}$) different representations, and therefore it can make different predictions for these two (predict that only the first exists, as the ground truth). Therefore, we have just shown that there exist two non-isomorphic triplets that TRIX can distinguish but ULTRA cannot. 

\revisionlog{
\textbf{\ourmethod is more expressive than InGram.} Consider the two triplets ($v_{7}$, $r_3$, $v_{10}$) and ($v_{7}$, $r_3$, $v_{11}$) in \Cref{fig:exampleingram}, which are non-isomorphic because $r_1 \neq r_2$, and assume they are not observable (i.e., they do not exist in the input graph). Assume, however, we want to predict the existence of these two, with the first one having label 1 (exists) and the second one having label 0 (does not exist). In practice, this means that the first one represents a missing link in an incomplete KG, while the other does not exist and should not be predicted as missing. The proof shows the impossibility of InGram to distinguish these two and therefore to make different predictions for the two.

In the relation graph of InGram, $r_1$ and $r_2$ are \revisionlog{automorphic}. This implies that InGram will give $r_1$ and $r_2$ the same relation embedding. Since this embedding is then used by InGram to obtain the entity embeddings, InGram will assign $v_{10}$ and $v_{11}$ the same entity embedding. Therefore, InGram will assign ($v_{7}$, $r_3$, $v_{10}$) and ($v_{7}$, $r_3$, $v_{11}$) the same representation, and therefore it will necessarily make the same prediction for these two (predict that they both exist or that they both do not exist, even though the ground truth tell us only one exist).

On the contrary, in the relation graph of \ourmethod, $r_1$ and $r_2$ are not \revisionlog{automorphic}. This implies that \ourmethod can give $r_1$ and $r_2$ different relation embeddings, and consequently $v_{10}$ and $v_{11}$ can get different entity embeddings. Therefore, \ourmethod can assign ($v_{7}$, $r_3$, $v_{10}$) and ($v_{7}$, $r_3$, $v_{11}$) different representations, and therefore it can make different predictions for these two (predict that only the first exists, as the ground truth). Therefore, we have just shown that there exist two non-isomorphic triplets that TRIX can distinguish but InGram cannot. 

}
\end{proof}


\label{second_proof}

\begin{figure}[t]
\centering
\includegraphics[width=1.0\textwidth]{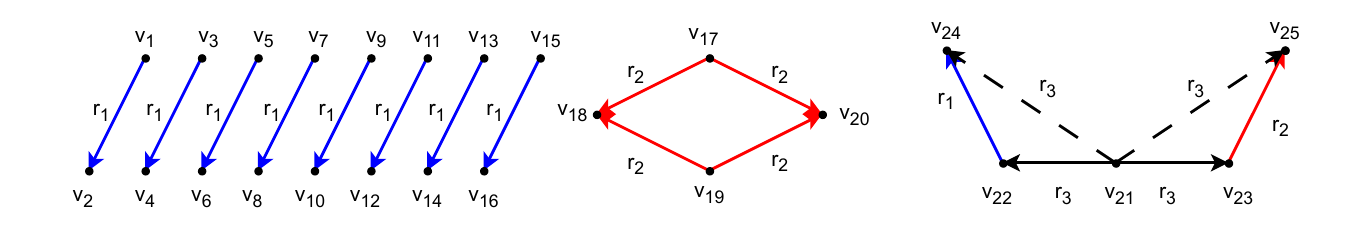}
\caption{A counter example \revisionlog{graph with 10 disconnected components} where there exist non-isomorphic triplets that can be distinguished by \ourmethod but that cannot be distinguished by ULTRA. In the figure, solid lines are edges observed in the knowledge graph and dashed lines are edges we want to predict whether existing or not. As discussed in \Cref{second_proof}, ULTRA always gives two dashed edges the same embedding while \ourmethod can distinguish them.}
\label{fig:example}
\end{figure}

\begin{figure}[t]
\centering
\includegraphics[width=0.7\textwidth]{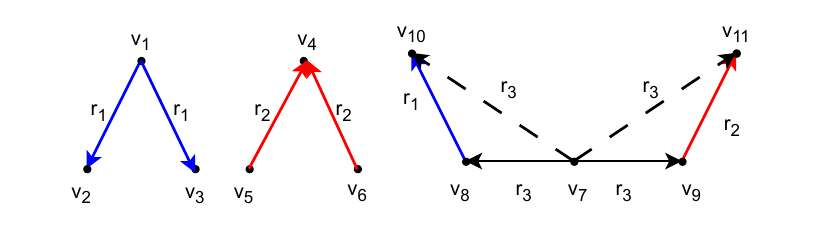}
\caption{\revisionlog{A counter example graph with 3 disconnected components where there exist non-isomorphic triplets that can be distinguished by \ourmethod but that cannot be distinguished by InGram. In the figure, solid lines are edges observed in the knowledge graph and dashed lines are edges we want to predict whether existing or not. As discussed in \Cref{second_proof}, InGram always gives two dashed edges the same embedding while \ourmethod can distinguish them.}}
\label{fig:exampleingram}
\end{figure}

\moreexpr*

\revision{\Cref{theo:asexpr} shows that any non-isomorphic triplet that can be distinguished by ULTRA \revisionlog{or InGram} can also be distinguished by \ourmethod. \Cref{theo:triplets} shows that there exist non-isomorphic triplets that can be distinguished by \ourmethod but that cannot be distinguished by ULTRA \revisionlog{or InGram}}. From \Cref{theo:asexpr,theo:triplets} directly follows that \ourmethod is more expressive than ULTRA \revisionlog{and InGram}.

\begin{proposition}
Define $\alpha=\max\limits_i (\max (\sum\limits_j \mathbbm{1}_{\mE_h[i, j] > 0}, \sum\limits_j \mathbbm{1}_{\mE_t[i, j] > 0}))$ which means $\alpha$ is the maximum number of unique relations one single entity connects to as the head or the tail. Then the relation adjacency matrix has $|R|$ nodes, $O(|V|\alpha^2)$ edges and $|V|$ relations.
\end{proposition}
\begin{proof}
    The dimension of relation adjacency matrix is $\mathbb{R}^{|R| \times |R| \times |V| \times 4}$. Since each node is affiliated with at most $\alpha$ relations and any two of these relations may generate four non-zero entries (one for head-to-head; one for tail-to-tail; one for head-to-tail; one for tail-to-head) in the relation graph, each node creates at most $4\alpha^2$ non-zero entries in the relation adjacency matrix, implying that the number of nodes is $O(|V|\alpha^2)$. Moreover, the total number of edges in the relation adjacency matrix is at most $4|V|\alpha^2$.
\end{proof}

Now, let us suppose \revision{there are $L$ rounds of iterative updates} and embedding dimension is $d$. The time complexity of \revision{entity embedding update} in \ourmethod is $O(L|E|d+L|V|d^2)$. The time complexity of \revision{relation embedding update} in \ourmethod is $O(L|V|\alpha^2d+L|R|d^2)$. Since $L$ and $d$ are always $O(1)$, the time complexity of \ourmethod is $O(|E|+|V|\alpha^2)$ for each query.

 The time complexity of relation feature propagation in ULTRA is $O(L|R|^2d+L|R|d^2)$. The time complexity of entity feature propagation in ULTRA is $O(L|E|d+L|V|d^2)$. Since $d$ and $L$ are always $O(1)$, the time complexity of ULTRA is $O(|E|+|V|+|R|^2)$ for each entity prediction query. However, it goes to $O(|E||R|+|V||R|+|R|^3)$ for each relation prediction query.

 For completeness, we report in the following the statistics of the original graph and the relation graph in three exemplary datasets as shown in \Cref{alpha}. Empirically, the number of edges in the relation graph is usually 4 to 10 times the number of edges in the original graph. This means for entity prediction query, the complexity of ULTRA is up to 10 times better than \ourmethod. However, in relation prediction, the complexity of \ourmethod is up to 20 times better than that of ULTRA since ULTRA needs to perform one forward pass for each relation in the relation set. 

 \begin{table*}[t]
\caption{Statistics of relation graph in pre-training datasets.}
\centering
\label{alpha}
\begin{tabular}{l|rrrr}
\toprule
 & $\alpha$ & \# relation & \# edge in original graph & \# edge in relation graph  \\
\midrule
WN18RR & 7 & 11 & 86835 & 314481 \\
FB15k237 & 53 & 237 & 272115 & 2247123\\ 
CoDExMedium & 19 & 51 & 185584 & 706534\\
\bottomrule
\end{tabular}
\end{table*}

\section{Computational Resources}\label{appx:compute}
We implemented \ourmethod using PyTorch~\citep{paszke2019pytorch} (offered under BSD-3 Clause license) and the PyTorch Geometric library~\citep{fey2019fast} (offered under MIT license) for efficient processing of graph-structured data. All experiments were conducted on NVIDIA RTX A5000,  NVIDIA RTX A6000, and NVIDIA GeForce RTX 4090 GPUs, and on the Google's Gemini API. For hyperparameter tuning and model selection, we used the Weights and Biases (wandb) library~\citep{wandb}.

\section{LLM Experiment Details}
\label{appx:LLMexp}

\subsection{Details for Task 1} \label{appx:T1}
\paragraph{Relation Prediction Prompt of Task 1} 
In the following task, you will be given background knowledge in the form of triplet (h, r, t) which means entity 'h' has relation 'r' with entity 't'. Then you will be asked some questions about the relationship between entities. Background knowledge: (Kris Kristofferson, occupation, guitarist); (Willow Smith, genre, indie pop);\dots What is the relationship between entity 'Gaspard Monge' and entity 'France'? Please choose one best answer from the following relations:|parent organization|studies|cause of death|architectural style|unmarried partner|industry|\dots|. You just need to give the relation and please do not give an explanation.

\paragraph{Entity Prediction Prompt of Task 1} 
In the following task, you will be given background knowledge in the form of triplet (h, r, t) which means entity 'h' has relation 'r' with entity 't'. Then you will be asked some questions about the relationship between entities. Background knowledge: (Kris Kristofferson, occupation, guitarist); (Willow Smith, genre, indie pop);\dots Predict the tail entity for triplet (Gaspard Monge, country of citizenship, ?). Please give the 10 most possible answers. You just need to give the names of the entities separated by commas and please do not give explanation.

\revisionlog{In the prompt for relation prediction for all the three tasks, all relations in the dataset are listed with '|' as the delimiter. For the sake of simplicity in the presentation, in the next subsections we use "\dots" to represent the rest of triplets and relations in the prompts.}

\Cref{llm-base} shows the Hits@1 of Gemini-1.5-pro on in-domain relation prediction task. \revision{\Cref{llm-base-entity} shows the Hits@10 of Gemini-1.5-flash on in-domain entity prediction task. We use Gemini-1.5-flash for entity prediction tasks for the sake of reducing costs of experiments.} We run the experiment 3 times with the same prompt in English to see if it can generate consistent answers. Gemini performs the task quite well and shows consistency across 3 runs. This indicates given background knowledge in the prompt, LLMs has the capacity to handle in-domain relation and entity predictions well. 


\begin{table*}[ht]
\centering
\caption{Task 1: In-domain LLM relation predictions Hits@1 on CoDEx-S.}
\begin{tabular}{lrrr |r}
\toprule
 & Run \#1  & Run \#2 & Run \#3  & Worst \\
Gemini-1.5-pro & 0.933 & 0.933 & 0.933 & 0.933\\
ULTRA & 0.820 & 0.820 & 0.820 & 0.820 \\
\ourmethod & {\bf 0.935} & {\bf 0.935} & {\bf 0.935} & {\bf 0.935} \\
\hline
\end{tabular}
\label{llm-base}
\end{table*}

\begin{table*}[ht]
\centering
\caption{Task 1: In-domain LLM entity predictions Hits@10 on CoDEx-S.}
\begin{tabular}{lrrr |r}
\toprule
 & Run \#1  & Run \#2 & Run \#3  & Worst \\
Gemini-1.5-flash & 0.308 & 0.308 & 0.308 & 0.308\\
ULTRA & 0.667 & 0.667 & 0.667 & 0.667 \\
\ourmethod & {\bf 0.670} & {\bf 0.670} & {\bf 0.670} & {\bf 0.670} \\
\hline
\end{tabular}
\label{llm-base-entity}
\end{table*}


\subsection{Details for Task 2} \label{appx:T2}
\paragraph{Relation Prediction Prompt of Task 2} In the following task, you will first be given background knowledge in the form of triplet (h, r, t) which means entity 'h' has relation 'r' with entity 't'. Then you will be asked some questions about the relationship between entities. Please notice that some words are replaced with metasyntactic words in the following paragraph. Background knowledge: (foo, baz, guitarist); (Willow Smith, genre, bar);\dots What is the relationship between entity 'foo' and entity 'bar'? Please choose one best answer from the following relation IDs:|parent organization|studies|quux|baz|\dots|. You just need to give the relation and please do not give an explanation.

\paragraph{Entity Prediction Prompt of Task 2} In the following task, you will first be given background knowledge in the form of triplet (h, r, t) which means entity 'h' has relation 'r' with entity 't'. Then you will be asked some questions about the relationship between entities. Please notice that some words are replaced with metasyntactic words in the following paragraph. Background knowledge: (foo, baz, guitarist); (Willow Smith, genre, bar);\dots Predict the tail entity for triplet (foo, garply, ?). Please give the 10 most possible answers. You just need to give the names of the entities separated by commas and please do not give explanation.

\Cref{llm-metasyntactic} shows the Hits@1 of Gemini-1.5-pro on out-domain relation prediction task. \revision{\Cref{llm-metasyntactic-entity} shows the Hits@10 of Gemini-1.5-flash on out-domain entity prediction task.} The neighbor entities of the head entity and the relations that connect the head entity with its neighbors are replaced with metasyntactic words. We run the experiment 3 times with different metasyntactic words but the underlying structural pattern is exactly the same as in the in-domain task. The results indicate the LLM can not do the out-of-domain task well. This means the LLM relied more on the known semantic description of the words instead of the structural pattern of the graph so that when there are new entities and relations, it can not perform inductive reasoning on them.

\begin{table*}[t]
\centering
\caption{Task 2: Out-of-domain LLM relation predictions Hits@1 on CoDEx-S. Effects of Metasyntactic Words on Relation Predictions on CoDEx-S.\label{llm-metasyntactic}
}
\begin{tabular}{lrrr |r}
\toprule
 & Run \#1  & Run \#2 & Run \#3  & Worst \\
 \midrule
Gemini-1.5-pro & 0.667 & 0.667 & 0.633 & 0.633\\
ULTRA & 0.820 & 0.820 & 0.820 & 0.820 \\
\ourmethod & {\bf 0.935} & {\bf 0.935} & {\bf 0.935} & {\bf 0.935} \\
\hline
\end{tabular}
\end{table*}

\begin{table*}[ht]
\centering
\caption{Task 2: Out-of-domain LLM entity predictions Hits@10 on CoDEx-S. Effects of Metasyntactic Words on Entity Predictions on CoDEx-S.\label{llm-metasyntactic-entity}
}
\begin{tabular}{lrrr |r}
\toprule
 & Run \#1  & Run \#2 & Run \#3  & Worst \\
 \midrule
Gemini-1.5-flash & 0.212  & 0.250 & 0.327 & 0.212\\
ULTRA & 0.667 & 0.667 & 0.667 & 0.667 \\
\ourmethod & {\bf 0.670} & {\bf 0.670} & {\bf 0.670} & {\bf 0.670} \\
\hline
\end{tabular}
\end{table*}

\subsection{Details for Task 3}\label{appx:T3}
\paragraph{Relation Prediction Prompt of Task 3} In the following task, entities and relations will be expressed with their IDs. You will first be given the mapping from entities to their IDs and the mapping from relations to their IDs. Then you will be given background knowledge in the form of triplet (h, r, t) which means entity 'h' has relation 'r' with entity 't'. Finally you will be asked some questions about the relationship between entities. Entity mapping: Mireille Darc is entity '8831'; Breton is entity '20512'; Tomas Tranströmer is entity '1641'\dots Relation mapping: located in the administrative terroritorial entity is relation '15' \dots Background knowledge: (15443, 16, 1093); (21198, 16, 9387); (14854, 8, 10218)\dots What is the relationship between entity '18127' and entity '1799'? Please choose one best answer from the following relation IDs:|45|48|27|35|\dots|. You just need to give the ID of that relation and please do not give an explanation.

\paragraph{Entity Prediction Prompt of Task 3} In the following task, entities and relations will be expressed with their IDs. You will first be given the mapping from entities to their IDs and the mapping from relations to their IDs. Then you will be given background knowledge in the form of triplet (h, r, t) which means entity 'h' has relation 'r' with entity 't'. Finally you will be asked some questions about the relationship between entities. Entity mapping: Mireille Darc is entity '8831'; Breton is entity '20512'; Tomas Tranströmer is entity '1641'\dots Relation mapping: located in the administrative terroritorial entity is relation '15' \dots Background knowledge: (15443, 16, 1093); (21198, 16, 9387); (14854, 8, 10218)\dots Predict the tail entity for triplet (18127, 45, ?). Please give the 10 most possible answers. You just need to give the IDs of the entities separated by commas and please do not give explanation.

\Cref{llm-equivarariance} shows the Hits@1 of Gemini-1.5-pro on relation prediction in Task 3. \revision{\Cref{llm-equivarariance-entity} shows the Hits@10 of Gemini-1.5-flash on entity prediction in Task 3.} The entities and relations are expressed as IDs. We run the experiment 3 times with permutated IDs but the underlying structural pattern is exactly the same as in the in-domain task. The results demonstrate that the LLM is very sensitive to ID permutation so that its performance is inconsistent across 3 permutations.

\begin{table*}[t]
\centering
\caption{Task 3: Out-of-domain LLM relation predictions Hits@1 on CoDEx-S. Effects of Input Permutations on Relation Predictions on CoDEx-S.}
\begin{tabular}{lrrr |r}
\toprule
 & Permutation \#1  & Permutation \#2 & Permutation \#3  & Worst \\
 \midrule
Gemini-1.5-pro & 0.346 & 0.731 & 0.615 & 0.346\\
ULTRA & 0.820 & 0.820 & 0.820 & 0.820 \\
\ourmethod & {\bf 0.935} & {\bf 0.935} & {\bf 0.935} & {\bf 0.935} \\
\hline
\end{tabular}
\label{llm-equivarariance}
\end{table*}

\begin{table*}[t]
\centering
\caption{Task 3: Out-of-domain LLM entity predictions Hits@10 on CoDEx-S. Effects of Input Permutations on Entity Predictions on CoDEx-S.}
\begin{tabular}{lrrr |r}
\toprule
 & Permutation \#1  & Permutation \#2 & Permutation \#3  & Worst \\
 \midrule
Gemini-1.5-flash & 0.212 & 0.250 & 0.231 & 0.212\\
ULTRA & 0.667 & 0.667 & 0.667 & 0.667 \\
\ourmethod & {\bf 0.670} & {\bf 0.670} & {\bf 0.670} & {\bf 0.670} \\
\hline
\end{tabular}
\label{llm-equivarariance-entity}
\end{table*}

\section{Datasets}
\label{appx:datasets}
The statistics of all 57 datasets used in the experiments in presented in Tables~\ref{tab:app_datasets_transd},\ref{tab:app_datasets_inde},\ref{tab:app_datasets_indr}. 
All datasets are publicly available under open licenses (MIT or CC-BY).

\begin{table*}[!t]
\centering
\caption{Transductive datasets (16) used in the experiments. Train, Valid, Test denote triples in the respective set. Task denotes the prediction task: \emph{h/t} is predicting both heads and tails, \emph{tails} is only predicting tails.}
\label{tab:app_datasets_transd}
\begin{adjustbox}{width=\textwidth}
\begin{tabular}{llrrrrrl}\toprule
Dataset & Reference &Entities &Rels &Train &Valid &Test &Task \\\midrule
CoDEx Small & \cite{safavi2020codex} &2034 &42 &32888 &1827 &1828 & h/t \\
WDsinger & \cite{dackgr} &10282 &135 &16142 &2163 &2203 & h/t  \\
FB15k237\_10 & \cite{dackgr} &11512 &237 &27211 &15624 &18150 & tails \\
FB15k237\_20 & \cite{dackgr} &13166 &237 &54423 &16963 &19776 & tails \\
FB15k237\_50 & \cite{dackgr} &14149 &237 &136057 &17449 &20324 & tails \\
FB15k237 & \cite{fb15k237} &14541 &237 &272115 &17535 &20466 & h/t \\
CoDEx Medium & \cite{safavi2020codex} &17050 &51 &185584 &10310 &10311 & h/t  \\
NELL23k & \cite{dackgr} &22925 &200 &25445 &4961 &4952 & h/t  \\
WN18RR & \cite{wn18rr} &40943 &11 &86835 &3034 &3134 & h/t \\
AristoV4 & \cite{ssl_rp} &44949 &1605 &242567 &20000 &20000 & h/t  \\
Hetionet & \cite{hetionet} &45158 &24 &2025177 &112510 &112510 & h/t  \\
NELL995 & \cite{nell995} &74536 &200 &149678 &543 &2818 & h/t \\
CoDEx Large & \cite{safavi2020codex} &77951 &69 &551193 &30622 &30622 & h/t \\
ConceptNet100k & \cite{cnet100k} &78334 &34 &100000 &1200 &1200 & h/t  \\
DBpedia100k & \cite{dbp100k} &99604 &470 &597572 &50000 &50000 & h/t  \\
YAGO310 & \cite{yago310} &123182 &37 &1079040 &5000 &5000 & h/t \\
\bottomrule
\end{tabular}
\end{adjustbox}
\end{table*}

\begin{table*}[!t]
\caption{Inductive entity $(e)$ datasets (18) used in the experiments. Triples denote the number of edges of the graph given at training, validation, or test. Valid and Test denote triples to be predicted in the validation and test sets in the respective validation and test graph.}
\label{tab:app_datasets_inde}
\begin{adjustbox}{width=\textwidth}
\begin{tabular}{lrrrrrrrrrr}\toprule
\multirow{2}{*}{Dataset} &\multirow{2}{*}{Rels} &\multicolumn{2}{c}{Training Graph} &\multicolumn{3}{c}{Validation Graph} &\multicolumn{3}{c}{Test Graph}  \\ \cmidrule(l){3-4} \cmidrule(l){5-7} \cmidrule(l){8-10} 
& &Entities &Triples &Entities &Triples &Valid  &Entities &Triples &Test  \\\midrule
FB v1~\cite{teru2020inductive} &180 &1594 &4245 &1594 &4245 &489 &1093 &1993 &411  \\
FB v2~\cite{teru2020inductive} &200 &2608 &9739 &2608 &9739 &1166 &1660 &4145 &947  \\
FB v3~\cite{teru2020inductive} &215 &3668 &17986 &3668 &17986 &2194 &2501 &7406 &1731  \\
FB v4~\cite{teru2020inductive} &219 &4707 &27203 &4707 &27203 &3352 &3051 &11714 &2840  \\
WN v1~\cite{teru2020inductive} &9 &2746 &5410 &2746 &5410 &630 &922 &1618 &373 \\
WN v2~\cite{teru2020inductive} &10 &6954 &15262 &6954 &15262 &1838 &2757 &4011 &852  \\
WN v3~\cite{teru2020inductive} &11 &12078 &25901 &12078 &25901 &3097 &5084 &6327 &1143 \\
WN v4~\cite{teru2020inductive} &9 &3861 &7940 &3861 &7940 &934 &7084 &12334 &2823  \\
NELL v1~\cite{teru2020inductive} &14 &3103 &4687 &3103 &4687 &414 &225 &833 &201 \\
NELL v2~\cite{teru2020inductive} &88 &2564 &8219 &2564 &8219 &922 &2086 &4586 &935 \\
NELL v3~\cite{teru2020inductive} &142 &4647 &16393 &4647 &16393 &1851 &3566 &8048 &1620 \\
NELL v4~\cite{teru2020inductive} &76 &2092 &7546 &2092 &7546 &876 &2795 &7073 &1447 \\
ILPC Small~\cite{ilpc} &48 &10230 &78616 &6653 &20960 &2908 &6653 &20960 &2902 \\
ILPC Large~\cite{ilpc} &65 &46626 &202446 &29246 &77044 &10179 &29246 &77044 &10184  \\
HM 1k~\cite{ham_bm} &11 &36237 &93364 &36311 &93364 &1771 &9899 &18638 &476 \\
HM 3k~\cite{ham_bm} &11 &32118 &71097 &32250 &71097 &1201 &19218 &38285 &1349 \\
HM 5k~\cite{ham_bm} &11 &28601 &57601 &28744 &57601 &900 &23792 &48425 &2124\\
IndigoBM~\cite{liu2021indigo} &229 &12721 &121601 &12797 &121601 &14121 &14775 &250195 &14904 \\
\bottomrule
\end{tabular}
\end{adjustbox}
\end{table*}

\begin{table*}[!t]
\caption{Inductive entity and relation $(e,r)$ datasets (23) used in the experiments. Triples denote the number of edges of the graph given at training, validation, or test. Valid and Test denote triples to be predicted in the validation and test sets in the respective validation and test graph.}
\label{tab:app_datasets_indr}
\begin{adjustbox}{width=\textwidth}
\begin{tabular}{lrrrrrrrrrrrrrr}\toprule
\multirow{2}{*}{Dataset} &\multicolumn{3}{c}{Training Graph} &\multicolumn{4}{c}{Validation Graph} &\multicolumn{4}{c}{Test Graph} \\ \cmidrule(l){2-4} \cmidrule(l){5-8} \cmidrule(l){9-12}
&Entities &Rels &Triples &Entities &Rels &Triples &Valid &Entities &Rels &Triples &Test \\\midrule
FB-25~\cite{lee2023ingram} &5190 &163 &91571 &4097 &216 &17147 &5716 &4097 &216 &17147 &5716  \\
FB-50~\cite{lee2023ingram} &5190 &153 &85375 &4445 &205 &11636 &3879 &4445 &205 &11636 &3879  \\
FB-75~\cite{lee2023ingram} &4659 &134 &62809 &2792 &186 &9316 &3106 &2792 &186 &9316 &3106  \\
FB-100~\cite{lee2023ingram} &4659 &134 &62809 &2624 &77 &6987 &2329 &2624 &77 &6987 &2329  \\
WK-25~\cite{lee2023ingram} &12659 &47 &41873 &3228 &74 &3391 &1130 &3228 &74 &3391 &1131  \\
WK-50~\cite{lee2023ingram} &12022 &72 &82481 &9328 &93 &9672 &3224 &9328 &93 &9672 &3225  \\
WK-75~\cite{lee2023ingram} &6853 &52 &28741 &2722 &65 &3430 &1143 &2722 &65 &3430 &1144  \\
WK-100~\cite{lee2023ingram} &9784 &67 &49875 &12136 &37 &13487 &4496 &12136 &37 &13487 &4496  \\
NL-0~\cite{lee2023ingram} &1814 &134 &7796 &2026 &112 &2287 &763 &2026 &112 &2287 &763  \\
NL-25~\cite{lee2023ingram} &4396 &106 &17578 &2146 &120 &2230 &743 &2146 &120 &2230 &744  \\
NL-50~\cite{lee2023ingram} &4396 &106 &17578 &2335 &119 &2576 &859 &2335 &119 &2576 &859  \\
NL-75~\cite{lee2023ingram} &2607 &96 &11058 &1578 &116 &1818 &606 &1578 &116 &1818 &607  \\
NL-100~\cite{lee2023ingram} &1258 &55 &7832 &1709 &53 &2378 &793 &1709 &53 &2378 &793  \\
\midrule
Metafam~\cite{mtdea} &1316 &28 &13821 &1316 &28 &13821 &590 &656 &28 &7257 &184 \\
FBNELL~\cite{mtdea} &4636 &100 &10275 &4636 &100 &10275 &1055 &4752 &183 &10685 &597  \\
Wiki MT1 tax~\cite{mtdea} &10000 &10 &17178 &10000 &10 &17178 &1908 &10000 &9 &16526 &1834  \\
Wiki MT1 health~\cite{mtdea} &10000 &7 &14371 &10000 &7 &14371 &1596 &10000 &7 &14110 &1566  \\
Wiki MT2 org~\cite{mtdea} &10000 &10 &23233 &10000 &10 &23233 &2581 &10000 &11 &21976 &2441  \\
Wiki MT2 sci~\cite{mtdea} &10000 &16 &16471 &10000 &16 &16471 &1830 &10000 &16 &14852 &1650  \\
Wiki MT3 art~\cite{mtdea} &10000 &45 &27262 &10000 &45 &27262 &3026 &10000 &45 &28023 &3113  \\
Wiki MT3 infra~\cite{mtdea} &10000 &24 &21990 &10000 &24 &21990 &2443 &10000 &27 &21646 &2405  \\
Wiki MT4 sci~\cite{mtdea} &10000 &42 &12576 &10000 &42 &12576 &1397 &10000 &42 &12516 &1388  \\
Wiki MT4 health~\cite{mtdea} &10000 &21 &15539 &10000 &21 &15539 &1725 &10000 &20 &15337 &1703  \\
\bottomrule
\end{tabular}
\end{adjustbox}
\end{table*}

\section{Detailed Experiment Results of Entity and Relation Prediction}
\label{appx:detailedresults}
\revision{
\subsection{Loss Function}
\label{loss_function}
\ourmethod is trained by minimizing the binary cross entropy loss over positive and negative triplets.

For entity prediction, the loss function is:
$$\text{Loss} = -\log p(h, r, t) - \sum\limits_{i=1}^{n} \frac{1}{n} \log(1 - p(h_i^\prime, r, t_i^\prime))$$
where ($h, r, t$) is the positive triplet and ($h_i^\prime, r, t_i^\prime$) is a negative triplet by corrupting the head or the tail. $n$ is the number of negative triplets per positive triplet.

For relation prediction, the loss function is:
$$\text{Loss} = -\log p(h, r, t) - \sum\limits_{i=1}^{n} \frac{1}{n} \log(1 - p(h, r_i^\prime, t))$$
where ($h, r, t$) is the positive triplet and ($h, r_i^\prime, t$) is a negative triplet by corrupting the relation. $n$ is the number of negative triplets per positive triplet.
}

\subsection{Ablation Study}
We conducted multiple experiments to gain deeper insights into the pre-training quality of \ourmethod and to quantify the impact of the proposed adjacency matrix and the iterative updating scheme on its performance. \revision{We compare the performance of (1) \ourmethod, (2) \ourmethod without iterative updates, and (3) \ourmethod without iterative updates and without our relation graph (using the relation graph of ULTRA).} We can not do test \ourmethod w/o proposed relation graph and w/ iterative updates because \revision{iterative updates are impossible with the prior relation graphs since entities are not included as edges connecting relations in prior relation graphs and thus the message passing on prior relation graphs would not use the entity embeddings}. The data presented in Table \ref{ablation} clearly demonstrates that both the proposed relation graph and the iterative update scheme significantly enhance the prediction performance.

\begin{table*}[t]
\centering
\caption{Zero-shot entity prediction results of \ourmethod, \ourmethod without proposed relation graph and \ourmethod without iterative update scheme.}
\label{ablation}
\begin{tabular}{lcc}
\toprule
 & MRR & Hits@10 \\
 \midrule
\ourmethod w/o proposed relation graph and w/o iterative updates & 0.356  & 0.508  \\
\ourmethod w/ proposed relation graph and w/o iterative updates & 0.361 & 0.518 \\
\ourmethod & \textbf{0.390} & \textbf{0.548} \\
\bottomrule
\end{tabular}
\end{table*}

\subsection{Detailed Results}
\Cref{entity_app} shows an an overview of \ourmethod performance improvement compared with ULTRA. \Cref{tab:app_zero_shot_ent} shows detailed zero-shot entity prediction MRR and Hits@10.  \Cref{tab:app_fine_tune_ent} shows detailed fine-tuned entity prediction MRR and Hits@10. \Cref{tab:app_zero_shot_rel} shows detailed zero-shot relation prediction MRR and Hits@1.  \Cref{tab:app_fine_tune_rel} shows detailed fine-tuned relation prediction MRR and Hits@1. From these tables we can conclude \ourmethod outperforms the baseline model in both entity and relation prediction. 

\begin{table*}[t]
\centering
\caption{An overview of \ourmethod performance improvement compared with ULTRA.}
\label{entity_app}
\begin{tabular}{lrrrr}
\toprule
Task & \multicolumn{2}{c}{Entity} & \multicolumn{2}{c}{Relation}  \\
\midrule
Improvement & zero-shot & finetuned & zero-shot & finetuned\\
\midrule
-2\% and below & 4 & 3 & 8 & 12 \\
-2\% to 0\%& 9 & 19 & 2 & 3 \\
0\% to 2\% & 15 & 24 & 6 & 7 \\
2\% and above & 26 & 11 & 38 & 35 \\ \midrule
Total & 54 & 57 & 54 & 57 \\
\bottomrule
\end{tabular}
\end{table*}

\begin{table*}[t]
\centering
\caption{Zero-shot entity prediction MRR and Hits@10 over 57 KGs from distinct domains.}
\label{tab:app_zero_shot_ent}
\begin{tabular}{lccccc}
\toprule
Dataset & ULTRA MRR & \ourmethod MRR & ULTRA Hits@10 & \ourmethod Hits@10 \\
\midrule
CoDEx Small & \textbf{0.472} & \textbf{0.472} & 0.667 & \textbf{0.670} \\
CoDEx Large & \textbf{0.338} & 0.335 & \textbf{0.469} & \textbf{0.469} \\
NELL-995 & 0.406 & \textbf{0.472} & 0.543 & \textbf{0.629} \\
YAGO 310 & \textbf{0.451} & 0.409 & 0.615 & \textbf{0.627} \\
WDsinger & 0.382 & \textbf{0.511} & 0.498 & \textbf{0.609} \\
NELL23k & 0.239 & \textbf{0.290} & 0.408 & \textbf{0.497} \\
FB15k237\_10 & \textbf{0.248} & 0.246 & \textbf{0.398} & 0.393 \\
FB15k237\_20 & \textbf{0.272} & 0.269 & \textbf{0.436} & 0.430 \\
FB15k237\_50 & \textbf{0.324} & 0.321 & \textbf{0.526} & 0.521 \\
DBpedia100k & 0.398 & \textbf{0.426} & 0.576 & \textbf{0.603} \\
AristoV4 & \textbf{0.182} & 0.181 & 0.282 & \textbf{0.286} \\
ConceptNet100k & 0.082 & \textbf{0.193} & 0.162 & \textbf{0.345} \\
Hetionet & 0.257 & \textbf{0.279} & 0.379 & \textbf{0.420} \\
FB-100 & \textbf{0.449} & 0.436 & \textbf{0.642} & 0.635 \\
FB-75 & \textbf{0.403} & 0.401 & 0.604 & \textbf{0.611} \\
FB-50 & \textbf{0.338} & 0.334 & 0.543 & \textbf{0.547} \\
FB-25 & 0.388 & \textbf{0.393} & 0.640 & \textbf{0.650} \\
WK-100 & 0.164 & \textbf{0.188} & 0.286 & \textbf{0.299} \\
WK-75 & 0.365 & \textbf{0.368} & \textbf{0.537} & 0.513 \\
WK-50 & \textbf{0.166} & \textbf{0.166} & \textbf{0.324} & 0.313 \\
WK-25 & \textbf{0.316} & 0.305 & \textbf{0.532} & 0.496 \\
NL-100 & 0.471 & \textbf{0.486} & 0.651 & \textbf{0.676} \\
NL-75 & \textbf{0.368} & 0.351 & \textbf{0.547} & 0.525 \\
NL-50 & \textbf{0.407} & 0.404 & \textbf{0.570} & 0.548 \\
NL-25 & \textbf{0.395} & 0.377 & 0.569 & \textbf{0.589} \\
NL-0 & 0.342 & \textbf{0.385} & 0.523 & \textbf{0.549} \\
HM 1k & 0.059 & \textbf{0.072} & 0.092 & \textbf{0.128} \\
HM 3k & 0.037 & \textbf{0.069} & 0.077 & \textbf{0.119} \\
HM 5k & 0.034 & \textbf{0.062} & 0.071 & \textbf{0.110} \\
HM Indigo & \textbf{0.440} & 0.436 & \textbf{0.648} & 0.645 \\
MT1 tax & 0.224 & \textbf{0.358} & 0.305 & \textbf{0.452} \\
MT1 health & 0.298 & \textbf{0.376} & 0.374 & \textbf{0.457} \\
MT2 org & \textbf{0.095} & 0.091 & \textbf{0.159} & 0.156 \\
MT2 sci & 0.258 & \textbf{0.323} & 0.354 & \textbf{0.465} \\
MT3 art & 0.259 & \textbf{0.284} & 0.402 & \textbf{0.441} \\
MT3 infra & 0.619 & \textbf{0.655} & 0.755 & \textbf{0.797} \\
MT4 sci & 0.274 & \textbf{0.290} & 0.449 & \textbf{0.460} \\
MT4 health & 0.624 & \textbf{0.677} & 0.737 & \textbf{0.775} \\
Metafam & 0.238 & \textbf{0.341} & 0.644 & \textbf{0.815} \\
FBNELL & \textbf{0.485} & 0.473 & 0.652 & \textbf{0.660} \\
WN-v1 & 0.648 & \textbf{0.699} & 0.768 & \textbf{0.791} \\
WN-v2 & 0.663 & \textbf{0.678} & 0.765 & \textbf{0.781} \\
WN-v3 & 0.376 & \textbf{0.418} & 0.476 & \textbf{0.541} \\
WN-v4 & 0.611 & \textbf{0.648} & 0.705 & \textbf{0.723} \\
FB-v1 & 0.498 & \textbf{0.515} & 0.656 & \textbf{0.682} \\
FB-v2 & 0.512 & \textbf{0.525} & 0.700 & \textbf{0.730} \\
FB-v3 & 0.491 & \textbf{0.501} & 0.654 & \textbf{0.669} \\
FB-v4 & 0.486 & \textbf{0.493} & 0.677 & \textbf{0.687} \\
NL-v1 & 0.785 & \textbf{0.806} & \textbf{0.913} & 0.898 \\
NL-v2 & 0.526 & \textbf{0.569} & 0.707 & \textbf{0.768} \\
NL-v3 & 0.515 & \textbf{0.558} & 0.702 & \textbf{0.743} \\
NL-v4 & 0.479 & \textbf{0.538} & 0.712 & \textbf{0.765} \\
ILPC Small & 0.302 & \textbf{0.303} & 0.443 & \textbf{0.455} \\
ILPC Large & 0.290 & \textbf{0.307} & 0.424 & \textbf{0.428} \\
\bottomrule
\end{tabular}
\end{table*}

\begin{table*}[t]
\centering
\caption{Finetuned entity prediction MRR and Hits@10 over 57 KGs from distinct domains.}
\label{tab:app_fine_tune_ent}
\begin{tabular}{lccccc}
\toprule
Dataset & ULTRA MRR & \ourmethod MRR & ULTRA Hits@10 & \ourmethod Hits@10 \\
\midrule
WN18RR &0.480 \tiny{± 0.000} &\textbf{0.514} \tiny{± 0.003} &\textbf{0.614} \tiny{± 0.000} &0.611 \tiny{± 0.005} \\
FB15K237 &\textbf{0.368} \tiny{± 0.000} &0.366 \tiny{± 0.002} &\textbf{0.564} \tiny{± 0.000} &0.559 \tiny{± 0.002} \\
CoDEx Medium &\textbf{0.372} \tiny{± 0.000} &0.365 \tiny{± 0.001} &\textbf{0.525} \tiny{± 0.000} &0.521 \tiny{± 0.001} \\
CoDEx Small &\textbf{0.490} \tiny{± 0.003} &0.484 \tiny{± 0.001} &\textbf{0.686} \tiny{± 0.003} &0.676 \tiny{± 0.003} \\
CoDEx Large &0.343 \tiny{± 0.002} &\textbf{0.348} \tiny{± 0.002} &0.478 \tiny{± 0.002} &\textbf{0.481} \tiny{± 0.002} \\
NELL-995 &\textbf{0.509} \tiny{± 0.013} &0.506 \tiny{± 0.030} &\textbf{0.660} \tiny{± 0.006} &0.648 \tiny{± 0.016} \\
YAGO 310 &\textbf{0.557} \tiny{± 0.009} &0.541 \tiny{± 0.050} &\textbf{0.710} \tiny{± 0.003} &0.702 \tiny{± 0.021} \\
WDsinger &0.417 \tiny{± 0.002} &\textbf{0.502} \tiny{± 0.001} &0.526 \tiny{± 0.002} &\textbf{0.620} \tiny{± 0.001} \\
NELL23k &0.268 \tiny{± 0.001} &\textbf{0.306} \tiny{± 0.010} &0.450 \tiny{± 0.001} &\textbf{0.536} \tiny{± 0.007} \\
FB15k237\_10 &\textbf{0.254} \tiny{± 0.001} &0.253 \tiny{± 0.001} &\textbf{0.411} \tiny{± 0.001} &0.408 \tiny{± 0.002} \\
FB15k237\_20 &\textbf{0.274} \tiny{± 0.001} &0.273 \tiny{± 0.001} &\textbf{0.445} \tiny{± 0.002} &0.441 \tiny{± 0.001} \\
FB15k237\_50 &\textbf{0.325} \tiny{± 0.002} &0.322 \tiny{± 0.001} &\textbf{0.528} \tiny{± 0.002} &0.522 \tiny{± 0.002} \\
DBpedia100k &0.436 \tiny{± 0.008} &\textbf{0.457} \tiny{± 0.026} &0.603 \tiny{± 0.006} &\textbf{0.619} \tiny{± 0.012} \\
AristoV4 &0.343 \tiny{± 0.006} &\textbf{0.345} \tiny{± 0.009} &0.496 \tiny{± 0.004} & \textbf{0.499} \tiny{± 0.010} \\
ConceptNet100k &0.310 \tiny{± 0.004} &\textbf{0.340} \tiny{± 0.008} &0.529 \tiny{± 0.007} &\textbf{0.564} \tiny{± 0.001} \\
Hetionet &\textbf{0.399} \tiny{± 0.005} &0.394 \tiny{± 0.004} &\textbf{0.538} \tiny{± 0.004} &0.534 \tiny{± 0.005} \\
WN-v1 &0.685 \tiny{± 0.003} &\textbf{0.705} \tiny{± 0.007} &0.793 \tiny{± 0.003} &\textbf{0.798} \tiny{± 0.005} \\
WN-v2 &0.679 \tiny{± 0.002} &\textbf{0.682} \tiny{± 0.004} &0.779 \tiny{± 0.003} &\textbf{0.780} \tiny{± 0.002} \\
WN-v3 &0.411 \tiny{± 0.008} &\textbf{0.425} \tiny{± 0.010} &\textbf{0.546} \tiny{± 0.006} &0.543 \tiny{± 0.006} \\
WN-v4 &0.614 \tiny{± 0.003} &\textbf{0.650} \tiny{± 0.002} &0.720 \tiny{± 0.001} &\textbf{0.722} \tiny{± 0.002} \\
FB-v1 &0.509 \tiny{± 0.002} &\textbf{0.515} \tiny{± 0.000} &0.670 \tiny{± 0.004} &\textbf{0.682} \tiny{± 0.000} \\
FB-v2 &0.524 \tiny{± 0.003} &\textbf{0.525} \tiny{± 0.000} &0.710 \tiny{± 0.004} &\textbf{0.730} \tiny{± 0.000} \\
FB-v3 &\textbf{0.504} \tiny{± 0.001} &0.501 \tiny{± 0.000} &0.663 \tiny{± 0.003} &\textbf{0.669} \tiny{± 0.000} \\
FB-v4 &\textbf{0.496} \tiny{± 0.001} &0.493 \tiny{± 0.000} &0.684 \tiny{± 0.001} &\textbf{0.687}± \tiny{0.000} \\
NL-v1 &0.757 \tiny{± 0.021} &\textbf{0.804} \tiny{± 0.007} &0.878 \tiny{± 0.035} &\textbf{0.899} \tiny{± 0.001} \\
NL-v2 &\textbf{0.575} \tiny{± 0.004} &0.571 \tiny{± 0.003} &0.761 \tiny{± 0.007} &\textbf{0.764} \tiny{± 0.006} \\
NL-v3 &0.563 \tiny{± 0.004} &\textbf{0.571} \tiny{± 0.007} &0.755 \tiny{± 0.006} &\textbf{0.759} \tiny{± 0.006} \\
NL-v4 &0.469 \tiny{± 0.020} &\textbf{0.551} \tiny{± 0.001} &0.733 \tiny{± 0.011} &\textbf{0.772} \tiny{± 0.004} \\
ILPC Small &\textbf{0.303} \tiny{± 0.001} &\textbf{0.303} \tiny{± 0.001} &0.453 \tiny{± 0.002} &\textbf{0.455} \tiny{± 0.001} \\
ILPC Large &0.308 \tiny{± 0.002} &\textbf{0.310} \tiny{± 0.002} &\textbf{0.431} \tiny{± 0.001} &\textbf{0.431} \tiny{± 0.003} \\
HM 1k &0.042 \tiny{± 0.002} &\textbf{0.072} \tiny{± 0.000} &0.100 \tiny{ ± 0.007} &\textbf{0.128} \tiny{± 0.000} \\
HM 3k &0.030 \tiny{± 0.002} &\textbf{0.069} \tiny{± 0.000} &0.090 \tiny{ ± 0.003} &\textbf{0.119} \tiny{± 0.000} \\
HM 5k &0.025 \tiny{± 0.001} &\textbf{0.074} \tiny{± 0.021} &0.068 \tiny{ ± 0.003} &\textbf{0.118} \tiny{± 0.013} \\
HM Indigo &0.432 \tiny{± 0.001} &\textbf{0.436} \tiny{± 0.000} &0.639 \tiny{ ± 0.002} &\textbf{0.645} \tiny{± 0.000} \\
FB-100 &\textbf{0.444} \tiny{± 0.003} &0.436 \tiny{± 0.001} &\textbf{0.643} \tiny{± 0.004} &0.633 \tiny{± 0.003} \\
FB-75 &0.400 \tiny{± 0.003} &\textbf{0.401} \tiny{± 0.000} &0.598 \tiny{± 0.004} &\textbf{0.611} \tiny{± 0.000} \\
FB-50 &\textbf{0.334} \tiny{± 0.002} &\textbf{0.334} \tiny{± 0.000} &0.538 \tiny{± 0.004} &\textbf{0.547} \tiny{± 0.000} \\
FB-25 &0.383 \tiny{± 0.001} &\textbf{0.393} \tiny{± 0.000} &0.635 \tiny{± 0.002} &\textbf{0.650} \tiny{± 0.000} \\
WK-100 &0.168 \tiny{± 0.005} &\textbf{0.188} \tiny{± 0.000} &0.286 \tiny{± 0.003} &\textbf{0.299} \tiny{± 0.000} \\
WK-75 &\textbf{0.380} \tiny{± 0.001} &0.368 \tiny{± 0.000} &\textbf{0.530} \tiny{± 0.009} &0.513 \tiny{± 0.000} \\
WK-50 &0.140 \tiny{± 0.010} &\textbf{0.166} \tiny{± 0.000} &0.280 \tiny{± 0.012} &\textbf{0.313} \tiny{± 0.000} \\
WK-25 &\textbf{0.321} \tiny{± 0.003} &0.300 \tiny{± 0.009} &\textbf{0.535} \tiny{± 0.007} &0.493 \tiny{± 0.006} \\
NL-100 &0.458 \tiny{± 0.012} &\textbf{0.482} \tiny{± 0.002} &0.684 \tiny{± 0.011} &\textbf{0.691} \tiny{± 0.001} \\
NL-75 &\textbf{0.374} \tiny{± 0.007} &0.351 \tiny{± 0.000} &\textbf{0.570} \tiny{± 0.005} &0.525 \tiny{± 0.000} \\
NL-50 &\textbf{0.418} \tiny{± 0.005} &0.405 \tiny{± 0.002} &\textbf{0.595} \tiny{± 0.005} &0.555 \tiny{± 0.012} \\
NL-25 &\textbf{0.407} \tiny{± 0.009} &0.377 \tiny{± 0.000} &\textbf{0.596} \tiny{± 0.012} &0.589 \tiny{± 0.000} \\
NL-0 &0.329 \tiny{± 0.010} &\textbf{0.385} \tiny{± 0.000} &\textbf{0.551} \tiny{± 0.012} &0.549 \tiny{± 0.000} \\
MT1 tax &0.330 \tiny{± 0.046} &\textbf{0.397} \tiny{± 0.001} &0.459 \tiny{± 0.056} &\textbf{0.508} \tiny{± 0.002} \\
MT1 health &\textbf{0.380} \tiny{± 0.002} &0.376 \tiny{± 0.000} &\textbf{0.467} \tiny{± 0.006} &0.457 \tiny{± 0.000} \\
MT2 org &\textbf{0.104} \tiny{± 0.001} &0.098 \tiny{± 0.002} &\textbf{0.170} \tiny{± 0.001} &0.162 \tiny{± 0.002} \\
MT2 sci &0.311 \tiny{± 0.010} &\textbf{0.331} \tiny{± 0.012} &0.451 \tiny{± 0.042} &\textbf{0.526} \tiny{± 0.005} \\
MT3 art &\textbf{0.306} \tiny{± 0.003} &0.289 \tiny{± 0.004} &\textbf{0.473} \tiny{± 0.003} &0.441 \tiny{± 0.001} \\
MT3 infra &0.657 \tiny{± 0.008} &\textbf{0.672} \tiny{± 0.003} &0.807 \tiny{± 0.007} &\textbf{0.810} \tiny{± 0.002} \\
MT4 sci &0.303 \tiny{± 0.007} &\textbf{0.305} \tiny{± 0.003} &0.478 \tiny{± 0.003} &\textbf{0.482} \tiny{± 0.001} \\
MT4 health &\textbf{0.704} \tiny{± 0.002} &0.702 \tiny{± 0.002} &\textbf{0.785} \tiny{± 0.002} &\textbf{0.785} \tiny{± 0.002} \\
Metafam &\textbf{0.997} \tiny{± 0.003} &\textbf{0.997} \tiny{± 0.003} &\textbf{1.000} \tiny{± 0.000} &\textbf{1.000} \tiny{± 0.000} \\
FBNELL &\textbf{0.481} \tiny{± 0.004} &0.478 \tiny{± 0.004} &\textbf{0.661} \tiny{± 0.011} &0.655 \tiny{± 0.012} \\
\bottomrule
\end{tabular}
\end{table*}

\begin{table*}[t]
\centering
\caption{Zero-shot relation prediction MRR and Hits@1 over 57 KGs from distinct domains.}
\label{tab:app_zero_shot_rel}
\begin{tabular}{lcccc}
\toprule
Dataset & ULTRA MRR & \ourmethod MRR & ULTRA Hits@1 & \ourmethod Hits@1 \\
\midrule
CoDEx Small & 0.900 & \textbf{0.961} & 0.820 & \textbf{0.935} \\
CoDEx Large & 0.892 & \textbf{0.902} & 0.824 & \textbf{0.837} \\
NELL-995 & \textbf{0.583} & 0.578 & 0.437 & \textbf{0.457} \\
YAGO 310 & 0.646 & \textbf{0.783} & 0.403 & \textbf{0.598} \\
WDsinger & 0.668 & \textbf{0.720} & 0.546 & \textbf{0.621} \\
NELL23k & 0.669 & \textbf{0.756} & 0.548 & \textbf{0.657} \\
FB15k237\_10 & 0.688 & \textbf{0.795} & 0.550 & \textbf{0.711} \\
FB15k237\_20 & 0.695 & \textbf{0.834} & 0.558 & \textbf{0.758} \\
FB15k237\_50 & 0.717 & \textbf{0.876} & 0.591 & \textbf{0.812} \\
DBpedia100k & 0.650 & \textbf{0.717} & 0.509 & \textbf{0.582} \\
AristoV4 & 0.254 & \textbf{0.389} & 0.201 & \textbf{0.265} \\
ConceptNet100k & 0.181 & \textbf{0.650} & 0.083 & \textbf{0.469} \\
Hetionet & 0.634 & \textbf{0.809} & 0.524 & \textbf{0.707} \\
WN-v1 & \textbf{0.836} & 0.792 & \textbf{0.740} & 0.613 \\
WN-v2 & \textbf{0.853} & 0.764 & \textbf{0.790} & 0.572 \\
WN-v3 & 0.707 & \textbf{0.741} & \textbf{0.577} & 0.568 \\
WN-v4 & \textbf{0.860} & 0.764 & \textbf{0.803} & 0.570 \\
FB-v1 & 0.646 & \textbf{0.705} & 0.523 & \textbf{0.599} \\
FB-v2 & 0.695 & \textbf{0.713} & 0.570 & \textbf{0.590} \\
FB-v3 & 0.679 & \textbf{0.742} & 0.553 & \textbf{0.644} \\
FB-v4 & 0.638 & \textbf{0.766} & 0.488 & \textbf{0.665} \\
NL-v1 & 0.636 & \textbf{0.657} & 0.358 & \textbf{0.453} \\
NL-v2 & 0.742 & \textbf{0.780} & 0.652 & \textbf{0.696} \\
NL-v3 & 0.669 & \textbf{0.725} & 0.544 & \textbf{0.612} \\
NL-v4 & 0.606 & \textbf{0.794} & 0.489 & \textbf{0.691} \\
ILPC Small & 0.905 & \textbf{0.919} & 0.843 & \textbf{0.872} \\
ILPC Large & 0.875 & \textbf{0.894} & 0.799 & \textbf{0.829} \\
HM 1k & 0.626 & \textbf{0.663} & \textbf{0.447} & 0.414 \\
HM 3k & 0.592 & \textbf{0.664} & \textbf{0.439} & 0.418 \\
HM 5k & 0.605 & \textbf{0.672} & \textbf{0.452} & 0.428 \\
HM Indigo & 0.681 & \textbf{0.852} & 0.559 & \textbf{0.765} \\
FB-100 & 0.830 & \textbf{0.921} & 0.728 & \textbf{0.880} \\
FB-75 & 0.698 & \textbf{0.822} & 0.555 & \textbf{0.747} \\
FB-50 & 0.696 & \textbf{0.780} & 0.575 & \textbf{0.699} \\
FB-25 & 0.687 & \textbf{0.805} & 0.565 & \textbf{0.724} \\
WK-100 & 0.887 & \textbf{0.907} & 0.812 & \textbf{0.869} \\
WK-75 & 0.911 & \textbf{0.916} & 0.875 & \textbf{0.883} \\
WK-50 & 0.865 & \textbf{0.868} & 0.793 & \textbf{0.818} \\
WK-25 & 0.857 & \textbf{0.881} & 0.760 & \textbf{0.823} \\
NL-100 & 0.743 & \textbf{0.884} & 0.564 & \textbf{0.796} \\
NL-75 & \textbf{0.795} & 0.788 & 0.692 & \textbf{0.699} \\
NL-50 & 0.680 & \textbf{0.755} & 0.569 & \textbf{0.636} \\
NL-25 & 0.688 & \textbf{0.742} & 0.562 & \textbf{0.614} \\
NL-0 & 0.632 & \textbf{0.658} & 0.502 & \textbf{0.519} \\
MT1 tax & \textbf{0.985} & 0.975 & \textbf{0.976} & 0.958 \\
MT1 health & 0.721 & \textbf{0.973} & 0.561 & \textbf{0.949} \\
MT2 org & 0.974 & \textbf{0.986} & 0.951 & \textbf{0.973} \\
MT2 sci & \textbf{0.976} & 0.964 & \textbf{0.961} & 0.941 \\
MT3 art & 0.881 & \textbf{0.885} & 0.798 & \textbf{0.825} \\
MT3 infra & \textbf{0.962} & 0.940 & \textbf{0.935} & 0.905 \\
MT4 sci & 0.933 & \textbf{0.966} & 0.891 & \textbf{0.944} \\
MT4 health & 0.826 & \textbf{0.937} & 0.719 & \textbf{0.898} \\
Metafam & 0.124 & \textbf{0.291} & 0.000 & \textbf{0.011} \\
FBNELL & 0.700 & \textbf{0.726} & 0.564 & \textbf{0.605} \\
\bottomrule
\end{tabular}
\end{table*}

\begin{table*}[t]
\centering
\caption{Finetuned relation prediction MRR and Hits@1 over 57 KGs from distinct domains.}
\label{tab:app_fine_tune_rel}
\begin{tabular}{lccccc}
\toprule
Dataset & ULTRA MRR & \ourmethod MRR & ULTRA Hits@1 & \ourmethod Hits@1 \\
\midrule
WN18RR & \textbf{0.914} \tiny{± 0.004} & 0.783 \tiny{± 0.009} & \textbf{0.871} \tiny{± 0.001} & 0.634 \tiny{± 0.007} \\
FB15K237 & 0.795 \tiny{± 0.017} & \textbf{0.924} \tiny{± 0.005} & 0.709 \tiny{± 0.025} & \textbf{0.870} \tiny{± 0.024} \\
CoDEx Medium & 0.919 \tiny{± 0.032} & \textbf{0.931} \tiny{± 0.001} & 0.870 \tiny{± 0.048} & \textbf{0.886} \tiny{± 0.001} \\
CoDEx Small & 0.942 \tiny{± 0.007} & \textbf{0.964} \tiny{± 0.002} & 0.900 \tiny{± 0.014} & \textbf{0.943} \tiny{± 0.002} \\
CoDEx Large & 0.907 \tiny{± 0.000} & \textbf{0.908} \tiny{± 0.003} & \textbf{0.850} \tiny{± 0.000} & 0.845 \tiny{± 0.004} \\
NELL-995 & \textbf{0.630} \tiny{± 0.000} & 0.578 \tiny{± 0.000} & \textbf{0.513} \tiny{± 0.000} & 0.457 \tiny{± 0.000} \\
YAGO 310 & \textbf{0.930} \tiny{± 0.002} & 0.826 \tiny{± 0.000} & \textbf{0.891} \tiny{± 0.004} & 0.666 \tiny{± 0.000} \\
WDsinger & \textbf{0.730} \tiny{± 0.012} & 0.721 \tiny{± 0.004} & 0.603 \tiny{± 0.020} & \textbf{0.627} \tiny{± 0.007} \\
NELL23k & 0.688 \tiny{± 0.008} & \textbf{0.755} \tiny{± 0.004} & 0.571 \tiny{± 0.009} & \textbf{0.658} \tiny{± 0.005} \\
FB15k237\_10 & 0.688 \tiny{± 0.000} & \textbf{0.795} \tiny{± 0.000} & 0.550 \tiny{± 0.000} & \textbf{0.711} \tiny{± 0.000} \\
FB15k237\_20 & 0.695 \tiny{± 0.000} & \textbf{0.846} \tiny{± 0.011} & 0.558 \tiny{± 0.000} & \textbf{0.778} \tiny{± 0.017} \\
FB15k237\_50 & 0.728 \tiny{± 0.013} & \textbf{0.903} \tiny{± 0.003} & 0.618 \tiny{± 0.027} & \textbf{0.858} \tiny{± 0.006} \\
DBpedia100k & 0.650 \tiny{± 0.000} & \textbf{0.780} \tiny{± 0.003} & 0.509 \tiny{± 0.000} & \textbf{0.665} \tiny{± 0.006} \\
AristoV4 & 0.254 \tiny{± 0.000} & \textbf{0.498} \tiny{± 0.002} & 0.201 \tiny{± 0.000} & \textbf{0.381} \tiny{± 0.002} \\
ConceptNet100k & 0.612 \tiny{± 0.000} & \textbf{0.712} \tiny{± 0.005} & 0.488 \tiny{± 0.000} & \textbf{0.551} \tiny{± 0.003} \\
Hetionet & 0.737 \tiny{± 0.031} & \textbf{0.922} \tiny{± 0.002} & 0.646 \tiny{± 0.041} & \textbf{0.862} \tiny{± 0.005} \\
WN-v1 & \textbf{0.844} \tiny{± 0.021} & 0.776 \tiny{± 0.021} & \textbf{0.754} \tiny{± 0.029} & 0.591 \tiny{± 0.034} \\
WN-v2 & \textbf{0.834} \tiny{± 0.008} & 0.765 \tiny{± 0.009} & \textbf{0.766} \tiny{± 0.013} & 0.574 \tiny{± 0.015} \\
WN-v3 & 0.707 \tiny{± 0.000} & \textbf{0.756} \tiny{± 0.044} & 0.577 \tiny{± 0.000} & \textbf{0.594} \tiny{± 0.064} \\
WN-v4 & \textbf{0.861} \tiny{± 0.005} & 0.804 \tiny{± 0.013} & \textbf{0.795} \tiny{± 0.007} & 0.651 \tiny{± 0.026} \\
FB-v1 & 0.650 \tiny{± 0.008} & \textbf{0.705} \tiny{± 0.000} & 0.513 \tiny{± 0.014} & \textbf{0.599} \tiny{± 0.000} \\
FB-v2 & 0.675 \tiny{± 0.035} & \textbf{0.713} \tiny{± 0.000} & 0.547 \tiny{± 0.040} & \textbf{0.590} \tiny{± 0.000} \\
FB-v3 & 0.677 \tiny{± 0.007} & \textbf{0.742} \tiny{± 0.000} & 0.556 \tiny{± 0.006} & \textbf{0.644} \tiny{± 0.000} \\
FB-v4 & 0.690 \tiny{± 0.026} & \textbf{0.766} \tiny{± 0.000} & 0.560 \tiny{± 0.035} & \textbf{0.665} \tiny{± 0.000} \\
NL-v1 & \textbf{0.719} \tiny{± 0.061} & 0.590 \tiny{± 0.036} & \textbf{0.504} \tiny{± 0.113} & 0.341 \tiny{± 0.066} \\
NL-v2 & 0.668 \tiny{± 0.064} & \textbf{0.811} \tiny{± 0.000} & 0.549 \tiny{± 0.090} & \textbf{0.740} \tiny{± 0.000} \\
NL-v3 & 0.646 \tiny{± 0.014} & \textbf{0.757} \tiny{± 0.004} & 0.484 \tiny{± 0.022} & \textbf{0.643} \tiny{± 0.009} \\
NL-v4 & 0.570 \tiny{± 0.030} & \textbf{0.822} \tiny{± 0.011} & 0.412 \tiny{± 0.056} & \textbf{0.735} \tiny{± 0.011} \\
ILPC Small & \textbf{0.922} \tiny{± 0.001} & 0.919 \tiny{± 0.000} & \textbf{0.876} \tiny{± 0.001} & 0.872 \tiny{± 0.000} \\
ILPC Large & 0.875 \tiny{± 0.000} & \textbf{0.894} \tiny{± 0.000} & 0.799 \tiny{± 0.000} & \textbf{0.829} \tiny{± 0.000} \\
HM 1k & 0.626 \tiny{± 0.000} & \textbf{0.663} \tiny{± 0.000} & \textbf{0.447} \tiny{± 0.000} & 0.414 \tiny{± 0.000} \\
HM 3k & 0.592 \tiny{± 0.000} & \textbf{0.664} \tiny{± 0.000} & \textbf{0.439} \tiny{± 0.000} & 0.418 \tiny{± 0.000} \\
HM 5k & 0.605 \tiny{± 0.000} & \textbf{0.672} \tiny{± 0.000} & \textbf{0.452} \tiny{± 0.000} & 0.428 \tiny{± 0.000} \\
HM Indigo & 0.726 \tiny{± 0.005} & \textbf{0.835} \tiny{± 0.002} & 0.614 \tiny{± 0.004} & \textbf{0.746} \tiny{± 0.003} \\
FB-100 & 0.851 \tiny{± 0.006} & \textbf{0.921} \tiny{± 0.000} & 0.769 \tiny{± 0.016} & \textbf{0.880} \tiny{± 0.000} \\
FB-75 & 0.754 \tiny{± 0.020} & \textbf{0.822} \tiny{± 0.000} & 0.638 \tiny{± 0.032} & \textbf{0.747} \tiny{± 0.000} \\
FB-50 & 0.696 \tiny{± 0.000} & \textbf{0.780} \tiny{± 0.000} & 0.575 \tiny{± 0.000} & \textbf{0.699} \tiny{± 0.000} \\
FB-25 & 0.684 \tiny{± 0.021} & \textbf{0.805} \tiny{± 0.000} & 0.563 \tiny{± 0.024} & \textbf{0.724} \tiny{± 0.000} \\
WK-100 & \textbf{0.924} \tiny{± 0.003} & 0.916 \tiny{± 0.001} & 0.879 \tiny{± 0.002} & \textbf{0.885} \tiny{± 0.003} \\
WK-75 & 0.911 \tiny{± 0.000} & \textbf{0.937} \tiny{± 0.003} & 0.875 \tiny{± 0.000} & \textbf{0.910} \tiny{± 0.003} \\
WK-50 & 0.865 \tiny{± 0.000} & \textbf{0.881} \tiny{± 0.007} & 0.793 \tiny{± 0.000} & \textbf{0.840} \tiny{± 0.014} \\
WK-25 & 0.897 \tiny{± 0.002} & \textbf{0.905} \tiny{± 0.007} & 0.834 \tiny{± 0.005} & \textbf{0.860} \tiny{± 0.011} \\
NL-100 & 0.803 \tiny{± 0.008} & \textbf{0.885} \tiny{± 0.005} & 0.678 \tiny{± 0.012} & \textbf{0.793} \tiny{± 0.008} \\
NL-75 & \textbf{0.795} \tiny{± 0.000} & 0.790 \tiny{± 0.000} & 0.678 \tiny{± 0.000} & \textbf{0.671} \tiny{± 0.000} \\
NL-50 & \textbf{0.808} \tiny{± 0.000} & 0.774 \tiny{± 0.000} & 0.704 \tiny{± 0.000} & \textbf{0.683} \tiny{± 0.000} \\
NL-25 & \textbf{0.737} \tiny{± 0.000} & 0.709 \tiny{± 0.000} & 0.622 \tiny{± 0.000} & \textbf{0.606} \tiny{± 0.000} \\
NL-0 & 0.632 \tiny{± 0.000} & \textbf{0.655} \tiny{± 0.006} & 0.502 \tiny{± 0.000} & \textbf{0.518} \tiny{± 0.002} \\
MT1 tax & 0.990 \tiny{± 0.001} & \textbf{0.995} \tiny{± 0.001} & 0.984 \tiny{± 0.001} & \textbf{0.990} \tiny{± 0.001} \\
MT1 health & 0.929 \tiny{± 0.044} & \textbf{0.973} \tiny{± 0.000} & 0.867 \tiny{± 0.087} & \textbf{0.949} \tiny{± 0.000} \\
MT2 org & 0.981 \tiny{± 0.014} & \textbf{0.987} \tiny{± 0.001} & 0.963 \tiny{± 0.027} & \textbf{0.978} \tiny{±0.001} \\
MT2 sci & 0.977 \tiny{± 0.001} & \textbf{0.990} \tiny{± 0.001} & 0.961 \tiny{± 0.001} & \textbf{0.984} \tiny{± 0.002} \\
MT3 art & \textbf{0.907} \tiny{± 0.012} & 0.887 \tiny{± 0.003} & 0.851 \tiny{± 0.024} & \textbf{0.828} \tiny{± 0.005} \\
MT3 infra & 0.966 \tiny{± 0.003} & \textbf{0.970} \tiny{± 0.001} & 0.947 \tiny{± 0.006} & \textbf{0.952} \tiny{± 0.003} \\
MT4 sci & 0.954 \tiny{± 0.002} & \textbf{0.972} \tiny{± 0.001} & 0.929 \tiny{± 0.002} & \textbf{0.952} \tiny{± 0.001} \\
MT4 health & 0.951 \tiny{± 0.006} & \textbf{0.986} \tiny{± 0.001} & 0.919 \tiny{± 0.010} & \textbf{0.979} \tiny{± 0.002} \\
Metafam & \textbf{0.368} \tiny{± 0.029} & 0.265 \tiny{± 0.044} & \textbf{0.112} \tiny{± 0.036} & 0.024 \tiny{± 0.022} \\
FBNELL & 0.720 \tiny{± 0.013} & \textbf{0.766} \tiny{± 0.004} & 0.576 \tiny{± 0.020} & \textbf{0.639} \tiny{± 0.006} \\
\bottomrule
\end{tabular}
\label{table:relation_prediction}
\end{table*}

\end{document}